\documentclass[]{article}
\usepackage{geometry}
\usepackage{layout}
\usepackage{soul}
\setstcolor{red}
\usepackage{nicefrac}
\usepackage{amsmath}
\usepackage{amsthm}
\usepackage{amssymb}
\usepackage{bm} 
\usepackage{bbm} 
\usepackage{mathtools}
\usepackage[algo2e, ruled,linesnumbered]{algorithm2e}
\usepackage{hyperref}
\usepackage{cleveref}

\usepackage{natbib}
\setcitestyle{authoryear,open={(},close={)}} 

\usepackage{longtable}

\usepackage{xspace}
\usepackage{upgreek}
\usepackage{dsfont}

\theoremstyle{plain}
\newtheorem{theorem}{Theorem}
\newtheorem{informal}{Main Result}
\newtheorem{proposition}{Proposition}
\newtheorem{lemma}{Lemma}

\newtheorem{corollary}{Corollary}
\newtheorem{problemstatement}{Problem Statement}

\theoremstyle{definition}
\newtheorem{definition}{Definition}
\newtheorem{assumption}{Assumption}
\newtheorem{innercustomassump}{Assumption}
\newenvironment{customassump}[1]
  {\renewcommand\theinnercustomassump{#1}\innercustomassump}
  {\endinnercustomassump}

\theoremstyle{remark}
\newtheorem{remark}{Remark}

\newcommand{\lp}{\left (} 
\newcommand{\rp}{\right )} 
\newcommand{\lb}{\left [}
\newcommand{\rb}{\right ]}

\newcommand{\mc}[1]{\mathcal{#1}}
\newcommand{\mbb}[1]{\mathbb{#1}}

\newcommand{\indi}[1]{\mathbbm{1}_{ \{#1\} }}

\newcommand{\X}{\mathcal{X}}
\newcommand{\Y}{\mathcal{Y}}

\DeclareMathOperator*{\argmax}{arg\,max}

\newcommand{\tbf}[1]{\textbf{#1}}

\makeatletter
\newcommand{\pushright}[1]{\ifmeasuring@#1\else\omit\hfill$\displaystyle#1$\fi\ignorespaces}
\newcommand{\pushleft}[1]{\ifmeasuring@#1\else\omit$\displaystyle#1$\hfill\fi\ignorespaces}
\makeatother

\usepackage{cleveref}

\expandafter\def\csname ver@etex.sty\endcsname{3000/12/31}

\usepackage{autonum}

\newcommand{\defined}{\coloneqq}
\newenvironment{proofoutline}
 {\proof[Proof outline]}
 {\endproof}

\newcommand{\ttt}[1]{\texttt{#1}}

\newcommand{\domain}{\mathcal{X}} 

\newcommand{\creg}[1][n]{\mathcal{R}_{#1}} 
 
\newcommand{\rkhs}[1][K]{\mc{H}_{#1}}

\newcommand{\holder}{H\"older }

\newcommand{\kmat}[1][\nu]{K_{#1}}

\newcommand{\matern}{Mat\'ern }

\newcommand{\igain}[1][n]{\gamma_{#1}}

\newcommand{\tOh}[1]{\widetilde{\mathcal{O}} \lp #1 \rp }

\newcommand{\algoref}[1]{Algorithm~\ref{#1}}

\newcommand{\cu}{\bar{c}}
\newcommand{\cl}{\underbar{c}}

\newcommand{\Dn}{\Delta_n}
\newcommand{\Dni}{\Delta_{n,i}}

\newcommand{\wni}{w_{n,i}}

\newcommand{\compLip}{\mc{C}_f^{(Lip)}}
\newcommand{\comp}{\underline{\mc{C}}_f}
\newcommand{\compupper}{\overline{\mc{C}}_f}

\newcommand{\posterior}{\ttt{ComputePosterior}\xspace}
\newcommand{\refine}{\ttt{RefinePartition}\xspace}

\newcommand{\gpucb}{\texttt{GP-UCB}\xspace}
\newcommand{\gpts}{\texttt{GP-TS}\xspace}
\newcommand{\pigpucb}{\texttt{$\pi-$GP-UCB}\xspace}
\newcommand{\kernelucb}{\texttt{SupKernelUCB}\xspace}

\newcommand{\dopt}{\tilde{d}_{\text{opt}}}

\newcommand{\bbR}{\mathbb{R}}

\usepackage{authblk}
\author[1]{Shubhanshu Shekhar\thanks{\texttt{shubhan2@andrew.cmu.edu}}}
\author[2]{Tara Javidi}
\affil[1]{Department of Statistics and Data Science, Carnegie Mellon University }
\affil[2]{Department of Electrical and Computer Engineering, UCSD}

\begin{document}

\title{Instance-Dependent Regret Analysis of Kernelized Bandits}
\date{}
\maketitle
\begin{abstract}
	We study the kernelized bandits problem,  which involves designing an adaptive strategy for querying a noisy zeroth-order-oracle to efficiently learn about the optimizer of an unknown function $f$ with a norm bounded by $M<\infty$ in a Reproducing Kernel Hilbert Space~(RKHS) associated with a positive definite kernel $K$.
	Prior results, working in a \emph{minimax framework}, have characterized the worst-case~(over all functions in the problem class) limits on regret achievable by \emph{any} algorithm, and have constructed algorithms with matching~(modulo polylogarithmic factors) worst-case performance. 
	These results suffer from two drawbacks. First, the minimax lower bound gives no information about the limits of regret achievable by algorithms on specific problem instances. Second, due to their worst-case nature, the existing upper bound analysis fails to adapt to easier problem instances within the function class. 
	Our work takes steps to address both these issues. 
	First, we derive \emph{instance-dependent} regret lower bounds for  algorithms  with uniformly~(over the function class) vanishing normalized cumulative regret.
    Our result, valid for all the practically relevant kernelized bandits algorithms, such as, \gpucb, \gpts and \kernelucb,  identifies a fundamental complexity measure associated with every problem instance.
	We then address the second issue, by proposing a new algorithm that  is minimax near-optimal while also demonstrating the ability to adapt to easier problem instances. 
\end{abstract}

\section{Introduction}
\label{sec:introduction}
    We consider the problem of optimizing a  function $f:\X = [0,1]^d \mapsto \bbR$ by adaptively gathering information about it via noisy zeroth-order-oracle queries.
    To make the problem tractable,  we assume that $f$ lies in the reproducing kernel Hilbert space (RKHS) associated with a given positive-definite kernel $K$, and its norm is bounded by a known constant $M<\infty$. 
    The function $f$ can  be accessed through noisy zeroth-order-oracle queries that return $y_x = f(x) + \eta_x$ at a query point $x \in \X$.  Given a total budget $n$,  the goal of an agent is to design an adaptive querying strategy, denoted by $\mc{A}$,  to  select a sequence of query points $(x_t)_{t=1}^n$, that incur a small cumulative regret $\mc{R}_n(\mc{A}, f)$, defined as 
    \begin{align}
        &\mc{R}_n (\mc{A}, f) \defined \sum_{t=1}^n f(x^*) - f(x_t), \qquad \text{where} \quad x^* \in \argmax_{ x\in \X} f(x). 
    \end{align}
    
    The cumulative regret forces the agent to address the exploration-exploitation trade-off and prevents it from querying too many points from the sub-optimal regions of the domain.

    The problem described above is referred to as the kernelized bandit or \emph{agnostic} Gaussian Process bandit problem. Prior theoretical works in this area have focused on establishing lower and upper bounds on the performance achievable by algorithms in the minimax setting. These results characterize of the worst-case limits, over all functions in the given RKHS,  of  performance achievable by any adaptive sampling algorithm~(see ~\Cref{subsec:contributions_bead} for details).
    Due to its worst case nature, the minimax framework does not account for the fact that there may exist functions that are easier to optimize than others in the same function class.
    As a result, the minimax regret bounds  do not accurately reflect the benefits of carefully designed adaptive strategies in the typical, non-adversarial, problem instances. In this paper, we take  a step towards addressing this issue and present the first \emph{instance-dependent} analysis for kernelized bandits. 
    
    To streamline our presentation, we focus  on the RKHSs corresponding  to the  \matern family of  kernels, denoted by $\{\rkhs[\kmat]: \nu > 0\}$. These kernels are most relevant for practical applications and also allow a graded control over the smoothness of its elements~(through a smoothness parameter~$\nu>0$) as described by \citet{stein2012interpolation}. Furthermore, we also restrict our attention primarily to analyzing the cumulative regret, and leave the extension of these to the pure exploration setting for future work. 
    
    The rest of this paper is organized as follows:   we discuss the limitations of the existing minimax analysis and present  an overview of our contributions in ~\Cref{subsec:contributions_bead}. We describe  related results in literature in ~\Cref{subsec:related_work_bead} to place our results in proper context. We formally state the problem and the required assumptions in~\Cref{sec:preliminaries_bead}, and in \Cref{sec:lower_bounds_bead} we derive the instance-dependent lower bounds of the regret achievable by a `good' class of algorithms~(this class includes all the existing algorithms analyzed in the literature including \gpucb, \gpts, \kernelucb). Finally, in \Cref{sec:algorithms_bead}, we propose a new algorithm that achieves the \emph{best of both worlds}: it matches the minimax lower bounds (up to polylogarithmic factors) for the \matern family of kernels in the worst case, and can also exploit some additional structure present in the given problem instance to achieve regret tighter than the minimax lower bound for those instances. 
    
    \subsection{Overview of Results}
    \label{subsec:contributions_bead}
    
        \sloppy For a given class of functions, $\mc{H}$, the minimax expected cumulative regret is defined as $\mc{R}_n^*(\mc{H}) \defined \inf_{ \mc{A}}\; \sup_{f \in \mc{H}} \; \mbb{E} \lb \mc{R}_n \lp f, \mc{A} \rp \rb$. 
        For the RKHS associated with \matern kernels with smoothness parameter $\nu>0$, prior work has established a minimax rate $\mc{R}_n^*(\rkhs[\kmat]) = \Theta \lp n^{(\nu+d)/(\nu+2d)} \rp$, ignoring polylogarithmic factors in $n$. 
        In particular, the worst-case algorithm independent lower bound of the order $\Omega \lp n^{{(\nu+d)}/{(2\nu+d)}} \rp$ for the \matern kernels with smoothness parameter $\nu>0$ was established by
        \citet{scarlett2017lower}.  On the other hand,  \citet{vakili2021information} recently derived tighter bounds on the mutual information gain (or equivalently the \emph{effective dimension}) associated with \matern kernels, that in turn implied that the \kernelucb algorithm of \citet{valko2013finite} matches (up to polylogarithmic terms) the above-stated lower bound, hence showing its near-optimality. 
        
        While the existing theoretical results provide a rather complete understanding of the worst-case performance limits for kernelized bandits, they suffer from two drawbacks:
        \begin{enumerate}
            \item  The existing lower bounds  are obtained by constructing a suitable subset of `hard' problem instances, and then demonstrating that there exists no algorithm that can perform well on all of those problems simultaneously. These results tell us that for \emph{any} algorithm, there exists at least one \emph{hard problem instance} on which that algorithm must incur a certain  regret. However, such results do not tell us what are the limits of performance for \emph{carefully designed, `good'} algorithms~(such as \gpucb; precise meaning of `good' is stated in Definition~\ref{def:uniform}) on the \emph{specific problem instance} presented to it.
           
            \item The existing analysis of most of the algorithms depend on global properties of the given function class, such as the dimension, kernel parameters and RKHS norm. Consider, for example, the \gpucb algorithm for which \citet{srinivas2012information} derived the following upper bound on the regret $\mc{R}_n  = \tOh{ \sqrt{n} \igain}$, where $\igain$ is the maximum information gain for the given kernel $K$~(formally defined in \eqref{eq:info_gain_bead}).  
            
            Since the maximum information gain $\igain$ is a property of the entire RKHS, the existing theory does not exploit any simplifying structure present in the specific problem instance. 
        \end{enumerate}
        
        Our main contributions make progress towards addressing the two issues stated above. In particular, we first derive an instance-dependent lower bound for algorithms with uniformly bounded normalized cumulative regret, and then we propose an algorithm which achieves near-optimal worst case performance, and can also exploit some additional structure present in problem instances. 
        
        First, we consider the following question: \emph{Suppose we are given an algorithm $\mc{A}$ that is known to have $\mc{O} \lp n^{a_0} \rp$ worst-case regret over all functions in the RKHS associated with \matern kernel $\kmat$~(denoted by $\rkhs[\kmat]$). Then, what values of expected regret can $\mc{A}$ achieve for a given function $f$ in $\rkhs[\kmat]$?} We answer this question by identifying a \emph{lower-complexity} term $\comp$, that characterizes the per-instance achievable limit. 
        
        \begin{informal}
        \label{informal:lower} 
            Suppose an algorithm $\mc{A}$ has a worst-case regret $\mc{O}\lp n^{a_0}\rp$ over functions in $\rkhs[\kmat]$. Then, for a given $f$ with $\|f\|_{\rkhs[\kmat]}<M$, we  have:
            \begin{align}
                \mbb{E}[\mc{R}_n(f, \mc{A})] = \Omega \lp \comp ( n^{-(1-a_0)}) \rp, 
            \end{align}
            where $\comp(\Delta) \defined \sum_{k \geq 0} m_k/(2^{k+2}\Delta)$ for any $\Delta>0$, and $m_k$ denotes the $2w_k = \mc{O}\lp (\Delta 2^k)^{1/\nu}\rp$ packing number of the annular set $\mc{Z}_k = \{x \in \X: 2^k \Delta \leq f(x^*) - f(x) < 2^{k+1}\Delta \}$. 
        \end{informal}
        
        To interpret the term $\comp$, let us deconstruct each element, $m_k/(2^{k+2}\Delta)$,  in its defining sum. Consider a ball~(denoted by $E$) of radius $w_k = \mc{O}((2^k\Delta)^{1/\nu})$ contained in the region $\mc{Z}_k$. By construction every point in $E$ is at most $2^{k+1} \Delta $ suboptimal for $f$. We first show that, if $\Delta > n^{-(1-a_0)}$, then $\mc{A}$ must spend at least $\Omega \lp 1/(2^{k+1}\Delta)^2\rp$ queries in the region $E$. This  implies that the total number of queries in the region $\mc{Z}_k$ is at least $\Omega(m_k/(2^{k+1}\Delta)^2)$, which in turn, implies that the regret  incurred by these queries is lower bounded by $\Omega \big( 2^k\Delta m_k/( 2^{k+1}\Delta )^2 \big) = \Omega \lp m_k/(2^{k+2}\Delta) \rp$. Further details of this argument are in~\Cref{sec:lower_bounds_bead}.

        The previous result can be specialized for minimax-optimal case by setting $a_0= a_{\nu}^* \defined (\nu+d)/(2\nu+d)$. This motivates our next question: \emph{Can we construct a minimax-optimal algorithm, that adapts, and incurs smaller regret on easier problem instances?} We address this question, by constructing a new algorithm $\mc{A}_1$ in~\Cref{sec:algorithms_bead}, for which we show the following~(see~\Cref{theorem:general-upper} for a more precise statement). 
        
        \begin{informal}
            We construct an algorithm, $\mc{A}_1$, that is minimax near-optimal for functions in $\rkhs[\kmat]$, and satisfies (with $a_{\nu}^* = (\nu+d)/(2\nu+d)$)
            \begin{align}
                \mbb{E}[\mc{R}_n\lp f, \mc{A}_1 \rp] = \widetilde{\mc{O}}\lp \compupper(n^{-(1-a_{\nu}^*)}) \rp, 
            \end{align}
            where $\compupper(\Delta)  \defined   \sum_{k \geq 0}  \widetilde{m}_k/( 2^{k \xi}  \Delta) $ for $\Delta>0$ and $\xi = \min\{1, \nu\}$. The term $\widetilde{m}_k$ is the $\widetilde{w}_k = \mc{O}\big( (2^{k\xi} \Delta )^{1/\xi}\big)$ packing number of the set $\widetilde{\mc{Z}}_k = \{x : f(x^*) - f(x) = \mc{O} ( 2^{k\xi} \Delta ) \}$. 
        \end{informal}
        
        Similar to $\comp$, the upper-complexity term $\compupper$ can also be interpreted in terms of the number of queries made by our proposed algorithm $\mc{A}_1$ in the regions $\widetilde{\mc{Z}}_k$.  However, in general, the term $\compupper$ is  larger than $\comp$. The is primarily due to the fact that the packing radius $\widetilde{w}_k$ used in $\compupper$ is smaller than the analogous term $w_k$ used in $\comp$, due to the presence of $\xi = \min\{1, \nu\}$ instead of $\nu$. Nevertheless, in~\Cref{theorem:bead_upper}, we  identify sufficient conditions under which $\compupper$ is strictly tighter than the minimax rate, thus demonstrating the ability to adapt to easier problem instances. 
        
        To summarize, our results imply that for a minimax-optimal algorithm, the per-instance regret on a function $f$ lies between $\comp(n^{-(1-a_{\nu}^*)})$ and $\compupper(n^{-(1-a_{\nu}^*)})$, where $a_{\nu}^* = (\nu+d)/(2\nu+d)$. 
        In the next subsection, we discuss in more details the existing theoretical results on kernelized bandits.

    \subsection{Related Work}
    \label{subsec:related_work_bead}

        \emph{Lower Bounds.} \citet{scarlett2017lower} characterized the fundamental limits on the \emph{worst-case} performance of \emph{any} kernelized bandit algorithm by obtaining the minimax lower bound on the regret for the RKHS of Squared-Exponential~(SE) and \matern kernels.
        For a given value of the query budget $n$, a \matern kernel $\kmat$ and bound on RKHS norm $0< M<\infty$, they constructed specific \emph{hard} collection of functions, denoted by $\{f_1, \ldots, f_{m_n}\} \subset \rkhs[\kmat](M)$ for some integer $m_n$. By a reduction to multiple hypothesis testing and an application of Fano's inequality, they then lower bounded the maximum expected regret of \emph{any} algorithm $\mc{A}$ on these functions by $\Omega\lp n^{(d+\nu)/(d+2\nu)} \rp$, which in turn implied the result on $\mc{R}_n^*(\rkhs[\kmat])$, since 
        \begin{align}
            \mc{R}_n^*\lp \rkhs[\kmat] \rp \defined \inf_{\mc{A}}\; \sup_{f \in \rkhs[\kmat](M)}\; \mbb{E}\lb \mc{R}_n \lp f, \mc{A} \rp \rb \geq \inf_{\mc{A}} \; \max_{1 \leq i \leq m_n}\; \mbb{E} \lb \mc{R}_n \lp f_i, \mc{A} \rp \rb = \Omega \lp n^{\frac{d+\nu}{d+2\nu}}\rp. 
        \end{align}
        
        However, the \emph{hard} functions employed by \citet{scarlett2017lower} are of the \emph{needle-in-haystack} type, and may not be representative of the typical functions belonging to the  RKHS.
        Thus, the corresponding regret lower bound may provide a pessimistic limit for   the achievable performance of the above algorithms on the specific problem instance encountered.
        We address this issue in \Cref{sec:lower_bounds_bead},  and  derive the first instance-dependent regret bound for this problem. 
        
        \emph{Beyond minimax analysis.} There exist some results in the related area of $\X$-armed bandits~(or Lipschitz bandits) which move beyond the worst case analysis towards instance-dependent bounds. The closest such work is by~\citet{bachoc2021instancedependent}, who obtain a precise characterization of the instance-dependent regret for algorithms with error certificates in the noiseless setting. 
        Similarly,~\citet{wang2019optimization} study the \emph{local minimax optimality} of \holder continuous functions, where they characterize the cumulative regret of functions that are close in $\sup$ norm to some~(possibly unknown) reference function, in terms of the properties of the reference function. 

        \emph{Upper Bounds.} The most commonly used kernelized bandit algorithm is \gpucb proposed by \citet{srinivas2012information} that was motivated by the Upper Confidence Bound~(UCB) strategy for multi-armed bandits (MABs)~\citep{auer2002finite}. 
        The \gpucb algorithm proceeds by selecting query points $(x_t)_{t \geq 1}$ that maximize the UCB of $f$ of the form $\mu_t(x) + \beta_t \sigma(t)(x)$ over the domain $\domain$ for suitable  factors $(\beta_t)_{t \geq 1}$.  For this algorithm, \citet{srinivas2012information} derived the following high-probability upper bound on $\mc{R}_n$ 
        \begin{align}
            \label{eq:info_regret_bead}
            \mc{R}_n = \tilde{\mc{O}}\lp \sqrt{n} \igain\rp.
        \end{align}
         In the above display $\igain$ is the maximum information gain associated with the kernel $K$, defined as 
        \begin{align}
            \label{eq:info_gain_bead}
            \igain \defined \max_{S \subset \domain : |S|=n} I(\bm{y}_S ;\; f), \quad \text{for} \quad f \sim GP(0, K),  
        \end{align}
        where $\bm{y}_S = (y_1, \ldots, y_n)$ is the vector of observations at points in $S = (x_1, \ldots, x_n)$ and $I(\bm{y}_S ; f)$ denotes the Shannon mutual information between the observations $\bm{y}_S$ and  the function $f$ assumed to be a sample from a zero-mean Gaussian Process $GP(0, K)$. Thus, $\igain$ denotes the maximum amount of information that can be gained about a function $f$ sampled from a zero-mean GP, through $n$ noisy observations. 
        To obtain explicit (in $n$) regret bounds from~\eqref{eq:info_regret_bead}, \citet{srinivas2012information} also derived upper bounds on $\igain$ for two important family of kernels, squared-exponential and \matern. More recently, \citet{vakili2021information} derived tighter bounds on $\igain$ for these families using a different approach than that employed by \citet{srinivas2012information}. In particular, for the \matern family, this implies the following regret bound for \gpucb: $\mc{R}_n =\tOh{ n^{ (3d/2 + \nu)/(d+2\nu)} }$ where $\nu$ is the smoothness parameter.  \citet{chowdhury2017kernelized} showed that the same upper bound is also achieved by the Thompson Sampling based algorithm, \gpts. This is a randomized strategy that sets the query point $x_t$ at time $t$ to a maximizer of a random sample (function) drawn from the posterior distribution on the function space based on the first $t-1$ observations. 
        
        The regret bound achieved by \gpucb and \gpts for \matern kernels, stated above,  is not sublinear for some ranges of smoothness parameter $\nu$~(i.e., $\nu<d/2$).  Recently,  \citet{janz2020bandit} addressed this issue by proposing an algorithm~(referred to as $\pi-$\gpucb), which  adaptively partitions the input space and fits independent GP models in each element of the partition. This structured approach to sampling yields an alternative bound on $\igain$,  and results in a tighter regret bound of the form $\creg = \tOh{n^{e_\nu}}$ for $\nu>1$, where $e_\nu = \nicefrac{d(2d+3) + 2\nu}{d(2d+4) + 4\nu}$.   Unlike the bounds of \citep{srinivas2012information, chowdhury2017kernelized}, this is sublinear for  $\nu>1$ and $d\geq 1$.
        
        \citet{valko2013finite} proposed the \kernelucb algorithm for this problem, that takes a different algorithmic approach  and proceeds by dividing the queried points into batches which consist of conditionally independent observations. This dependence structure among the points allows the use of simple Azuma's inequality for constructing tight confidence intervals. This is in contrast to the analysis of \gpucb, in which the complex dependence structure among the observations requires use of stronger martingale inequalities, and results in wider confidence intervals. In particular, \citet{valko2013finite} showed that the \kernelucb algorithm achieves a regret bound of $\mc{R}_n = \mc{O} \lp \sqrt{n \igain}\rp$. Note that this is tighter than the corresponding bound for \gpucb~(and \gpts) algorithm by a factor of $\sqrt{\igain}$.   By plugging in the recently derived bounds on $\igain$ for  \matern kernels by \citet{vakili2021information}, this implies that the \kernelucb algorithm  achieves a regret bound $\mc{R}_n = \tOh{n^{(d+\nu)/(d+2\nu)}}$, which matches algorithm independent lower bounds derived by \citet{scarlett2017lower} and \citet{cai2020lower}. 
        
        As stated earlier, all the upper bounds of the existing algorithms in literature depend on the term $\igain$ which is a global property of the function class. Hence, these results do not distinguish between easy and hard problem instances lying in the same class. Our proposed algorithm described in \Cref{sec:algorithms_bead}, in contrast, can exploit some additional structure present in a given problem instance while also matching the best known worst case performance (i.e., the bound achieved by \kernelucb). 


\section{Preliminaries}
\label{sec:preliminaries_bead}
    We present the formal definitions of several important terms in \Cref{subsec:definitions_bead}, and then describe the assumptions and the formal problem statement in \Cref{subsec:problem_statement}.  

    \subsection{Definitions}
    \label{subsec:definitions_bead}
        We begin with the definition of a positive definite kernel, that will then be used in defining an RKHS. 
        \begin{definition}[Positive Definite Kernel]
        \label{def:positive_definite}
        For a non-empty set $\mc{X}$, a symmetric function $K:\X \times \X \mapsto [0, \infty)$ is called a positive-definite kernel, if for any $m \in \mbb{N}$,  any $x_1, \ldots, x_m \in \X$ and $c_1, c_2, \ldots, c_m \in \mbb{R}$, the following is true:  $\sum_{i=1}^m \sum_{j=1}^m c_i K(x_i, x_j) c_j \geq 0$. 
        \end{definition}
        
        In this paper, we will focus primarily on a family of kernels, referred to the \matern family, that are parameterized by a smoothness parameter $\nu>0$. 
        
        \begin{definition}[\matern kernels]
        \label{def:matern_family}
        For $\nu>0$ and $\theta>0$, the \matern kernel $\kmat: \X \times \X \mapsto \mbb{R}$ is defined as 
        \begin{align}
            \label{eq:matern_def}
            \kmat \lp x, z \rp = \frac{1}{2^{\nu-1} \Gamma(\nu)} \lp \frac{\sqrt{2\nu} \|x-z\|}{\theta} \rp^{\nu} J_\nu \lp  \frac{\sqrt{2\nu} \|x-z\|}{\theta} \rp, 
        \end{align}
        where $J_\alpha$ denotes the modified Bessel function of the second kind of order $\alpha$. 
        \end{definition}
        
        The RKHS associated with \matern kernels  consist of functions with a `finite degree of smoothness'~\citep{kanagawa2018gaussian} as opposed to the infinitely differentiable functions lying the RKHS associated with the SE kernels.  Due to this property, the \matern kernels are most commonly used in practical problems~\citep[\S~1.7]{stein2012interpolation} as they provide a reasonable trade-off between analytical tractability and representation power. 
        
        We now present a formal definition of the RKHS associated with a positive definite kernel $K$. 

        \begin{definition}[RKHS]
        \label{def:rkhs} 
        For a nonempty set $\mc{X}$ and a positive-definite kernel $K$, the RKHS associated with $K$, denoted by $\rkhs$, is defined as the Hilbert space of functions on $\X$ with an inner product $\langle \cdot, \cdot \rangle$ satisfying the following: \tbf{(i)} for all $x \in \X$, the function $K(\cdot, x) \in \rkhs$, and \tbf{(ii)} for all $x \in \X$ and $g \in \rkhs$, we have $g(x) = \langle g, K(\cdot, x) \rangle$. 
        \end{definition}
        
        The equality $g(x) = \langle g, K(\cdot, x) \rangle$ is referred to as  the \emph{reproducing property} which lends the name to the RKHS. We next introduce the definition of Gaussian Processes~(GPs) that are often used as a surrogate model for estimating functions lying in an RKHS. 
        
        \begin{definition}[Gaussian Processes]
        \label{def:gaussian_processes_bead}
        For a positive definite kernel $K:\X \times \X \mapsto \mbb{R}$, we use $GP(0, K)$ to represent a stochastic process indexed by $\X$, denoted by $\{Z_x : x \in \X\}$, such that for any $m \in \mbb{N}$ and $x_1, \ldots, x_m \in \X$, the random vector $[Z_{x_1}, \ldots, Z_{x_m}] \sim N(0, \Sigma_m)$ with $\Sigma_m = [K(x_i, x_j)]_{1\leq i, j\leq m}$. 
        \end{definition}

        Finally, we introduce the formal definition of an adaptive querying strategy. 
        \begin{definition}[Adaptive Strategy]
        \label{def:adaptive_strategy}
        An adaptive querying strategy $\mc{A}$ consists of a sequence of mappings $(A_t)_{t=1}^{\infty}$ where $A_t : \lp \X \times \Y  \rp^{t-1} \times \mc{U}^{t} \mapsto \mc{X}$, where $\mc{U} = [0,1]$ represents the range of additional randomness (used in randomized algorithms such as \gpts).  
        \end{definition}

        \begin{definition}[Induced Probability Measure]
        \label{def:induced_probability}
            An adaptive sampling strategy $\mc{A}$ and a function $f$ induces a probability measure $\mbb{P}_{f, \mc{A}}$~(henceforth abbreviated as $\mbb{P}_f$) on the measurable space $\lp \Omega, \mc{F} \rp$, with $\Omega = \lp \X \times \Y \rp^n$ and $\mc{F}$ representing the Borel $\sigma-$algebra on $\Omega$. This measure assigns the probabilities to events $E = \prod_{t=1}^n \lp E_{t,X} \times E_{t,Y} \rp$ where $E_{t,X} \in \mc{B}_{\X}$ and $E_{t,Y} \in \mc{B}_{\Y}$. 
            \begin{align}
                \mbb{P}_f \big( (X^n, Y^n) \in E \big) = \prod_{t=1}^n P_f \lp Y_t \in E_{t,Y} \vert X^t, Y^{t-1} \rp \, P_{\mc{A}}\lp  X_t \in E_{t,X} | X^{t-1}, Y^{t-1} \rp. 
            \end{align}
            Here $\mbb{P}_f$ represents the noisy zeroth-order-oracle and $\mbb{P}_{\mc{A}}$ is determined by the sampling strategy.
        \end{definition}
    
        \emph{Notations.} We end this section, by listing some of the notations that will be used in the rest of the paper. As mentioned earlier, $\X = [0,1]^d$ for some $d \geq 1$ represents the domain and $f:\X \mapsto \mbb{R}$ is the unknown objective function. We will represent the set of optimal points of $f$ with $\X^*(f) = \{x \in \X : f(x) = \max_{x' \in \X} f(x') \}$. In the sequel, we will suppress the $f$ dependence of $\X^*(f)$ and only use $\X^*$. For any $x \in \X$ and $r>0$, we use $B(x,r)$ to denote the $\ell_2$ open ball $\{z \in \X: \|z-x\|_2 < r\}$. 
        
    \subsection{Problem Statement}
    \label{subsec:problem_statement}

        As stated in the introduction, we consider the problem of optimizing a black-box function $f: \X \mapsto \mathbb{R}$, where $\X = [0,1]^d$ that is assumed to have bounded norm in the RKHS of a given kernel $K$. This problem is usually studied under the following assumptions.
        
        \begin{assumption}
        \label{assump:rkhs} 
            We assume that the function $f$ lies in the RKHS associated with a positive-definite kernel $K$, denoted by  $\rkhs$. Furthermore, we assume that $\|f\|_{\rkhs} \leq M$ for some known constant $M< \infty$. 
        \end{assumption}
        
        The above assumption is standard in the kernelized bandits literature, and informally it states that the unknown function has \emph{low complexity}, where the complexity is quantified in terms of the RKHS norm. 

        \begin{customassump}{2a}
        \label{assump:noise-lower}
            We assume that the agent can access the objective function $f$ via noisy zeroth order oracle queries. More specifically, the oracle returns $y(x) = f(x) + \eta(x)$ for a query point $x \in \X$, where $\eta(x)$ i.i.d. sequence of $N(0, \sigma^2)$ random variables. 
        \end{customassump}
        
        \begin{customassump}{2b}\label{assump:noise-upper}
            We assume that the zeroth order oracle queries returns $y(x_t) = f(x_t) + \eta(x_t)$ for a query point $x \in \X$, where $\eta(x)$ i.i.d. sequence of $\sigma^2$-sub-Gaussian random variables. 
        \end{customassump}
        
        These two assumptions state that the observation noise has light tails, and hence we can construct tight confidence intervals for the unknown functions based on the noisy observations. We will use~\Cref{assump:noise-lower} in the statement of our lower bound, while the more general~\Cref{assump:noise-upper} will be used to state the upper bound result. 
        Note that imposing the $N(0, \sigma^2)$  requirement for stating the lower bounds is primarily to simplify the presentation. This is further discussed in Remark~\ref{remark:gaussian_noise} in Appendix~\ref{appendix:general_lower2}, and relies on the fact that we can obtain closed form expressions for KL-divergence involving Gaussian random variables.

        We end this section with a formal problem statement. 
        \begin{problemstatement}
        \label{problem_statement}
            Suppose Assumptions~\ref{assump:rkhs}~and~either~\ref{assump:noise-lower}~or~\ref{assump:noise-upper}  hold with known values of $M$ and $\sigma^2$. Then, given a total querying budget $n$,  design an adaptive strategy $\mc{A}$ to select query points $x_1, \ldots, x_n$, which incur a small cumulative regret $\mc{R}_n \lp \mc{A}, f \rp \defined \sum_{t=1}^n f(x^*) - f(x_t)$. 
        \end{problemstatement}


\section{Lower Bounds}
\label{sec:lower_bounds_bead}

    In order to derive instance-dependent bounds on the regret, we need to restrict our attention to `good' algorithms  which perform well for all elements the given problem class. We beign by presenting  the precise definition of this term below, which is motivated by a similar definition of consistent policies used in multi-armed bandits~\citep[see][Definition~16.1]{lattimore2020bandit}. 

    \begin{definition}[$a_0$-consistent]
    \label{def:uniform}
        We say an algorithm $\mc{A}$ is $a_0$-consistent over a given function class $\mc{F}$, if for all $a>a_0$ and $f \in \mc{F}$, the following holds: 
        \begin{align}
        \label{eq:uniform_def}
            \lim_{n \to \infty}\; \frac{ \mbb{E} \lb \mc{R}_n \lp \mc{A}, f \rp \rb}{ n^a } = 0. 
        \end{align}
    \end{definition}

    \begin{remark}
    \label{remark:uniform_rate_algorithms}
        Note that when $\mc{F}$ is the RKHS associated with a \matern kernel, then all the existing algorithms discussed in \Cref{subsec:related_work_bead} satisfy the condition above with some $a_0 \leq 1$. In particular,  this condition is satisfied by \gpucb and \gpts   with $a_0 = \min \{1, (\nu+3d/2)/(2\nu+d) \}$,  by \pigpucb with $a_0 = ( d(2d+3) + 2\nu)/(d(2d+4) + 4\nu) $ and by \kernelucb with $a_0 = (\nu+d)/(2\nu+d)$. 
    \end{remark}

    The above restriction on the class of algorithms is necessary to obtain instance-dependent regret bounds. Otherwise, for every problem instance $f$, there exists a trivial algorithm (that always queries $x^* \in \argmax_{x \in\X} f(x)$) that incurs zero regret on $f$, but incurs a linear regret for all functions for which $x^*$ is strictly suboptimal. 

    Next, we introduce the definition of a bump function used in \citep[Lemma~4]{cai2020lower} that will be used to construct the local perturbations~(see~\Cref{subsubsec:perturbation_general}) in our lower bound proof. 

    \begin{definition}[bump function $g$]
    \label{def:bump}
        Define the function $g = \exp \lp 1 - \frac{\|x\|^2}{1 - \|x\|^2} \rp \indi{\|x\|<1}$,  which satisfies the following properties:  
        \begin{itemize}
            \item $g$ is supported on the ball $B(0, 1)$. 
            \item $\sup_{x \in B(0,1)} \; g(x) = g(0) = 1$. 
            \item $\|g\|_{\rkhs[\kmat]} \defined M_\nu < \infty$ for some constant $M_\nu$ depending on $\nu$. 
            \item if $\tilde{g}(\cdot) = g(\frac{\cdot}{w})$ for some $w>0$, then $\|\tilde{g}\|_{\rkhs[\kmat]} \leq (1/w)^\nu M_\nu$. 
        \end{itemize}
    \end{definition}

    In \Cref{subsec:general_lower}, we first present a general lower bound that bounds the regret achievable by an algorithm on a given function, in terms of a complexity term that informally depends on the volume of the near-optimal regions  of of the input space for a given function $f$, as well as the exponent $a_0$ of the uniform regret condition for $a_0-$consistent algorithms. We then specialize this result for a smaller class of functions that also satisfy a local growth condition~(Assumption~\ref{assump:growth}) in \Cref{subsec:growth_lower} to get explicit in $n$ regret lower bounds. 
    
    \subsection{Overview of the argument}
    \label{subsec:warmup}
        We now present an informal description of the key ideas involved in the obtaining the main lower bound that will be formally stated as Theorem~\ref{theorem:general_lower} in \Cref{subsec:general_lower}.
    
        Suppose $f$ is a function lying in $\rkhs[\kmat](M)$ and let $\mc{A}$ be an $a_0$-consistent algorithm for this family of functions. Suppose $E_1, E_2, \ldots, E_m$ are $m$ disjoint subsets of the input space $\domain$ for some $m \geq 1$, with the property that 
            \begin{align}
            \label{eq:subopt-1}
                f(x) \leq f(x^*) - \Delta_i, \quad \text{for all}\; x \in E_i, \quad \text{for all} \; i \in [m]. 
            \end{align}
        Now if $N_i$ denotes the (random) number of times the algorithm $\mc{A}$ queries points in the region $E_i$ in $n$ rounds, then we immediately have the following regret lower bound. 
            \begin{align}
                \label{eq:regret_decomposition-1}
                \mbb{E}_f[\mc{R}_n\lp \mc{A}, f \rp ] & \geq \sum_{i=1}^m \Delta_i \mbb{E}_{f}\lb N_i \rb. 
            \end{align}
        The expression in~\eqref{eq:regret_decomposition-1} suggests that one way of lower bounding the regret incurred by $\mc{A}$ on the function $f$ is to lower bound the expected number of samples it allocates in these suboptimal regions, $\mbb{E}_{f}[N_i]$ for $1 \leq i \leq m$. 
        We approach this task by considering functions that are slightly perturbed versions of $f$, denoted by $\tilde{f}_i$,  that also lie in the same RKHS. The perturbed function $\tilde{f}_i$ shall differ from $f$ only in the region $E_i$, but this difference should be substantial enough to ensure that the maximizer of $\tilde{f}_i$ lies in $E_i$. This fact makes $\tilde{f}_i$ operationally distinct from $f$,  for which the region $E_i$ is at least $\Delta_i$-suboptimal as assumed in~\eqref{eq:subopt-1}. 
        Now, since the algorithm $\mc{A}$ is $a_0$-consistent for the given function class, it must achieve $o(n^a)$ regret (for any $a>a_0$) for all such functions~(and in particular, for $f$ and all of $\tilde{f}_i$). These two facts will enable us to  bound the number of samples that the algorithm $\mc{A}$ must spend (on an average) in the suboptimal region $E_i$ when $f$ is the true function.
        
         \begin{figure*}[hbt]
            \centering
            \begin{subfigure}[t]{0.5\textwidth}
                \centering
                \includegraphics[height=2.5in, width=\textwidth]{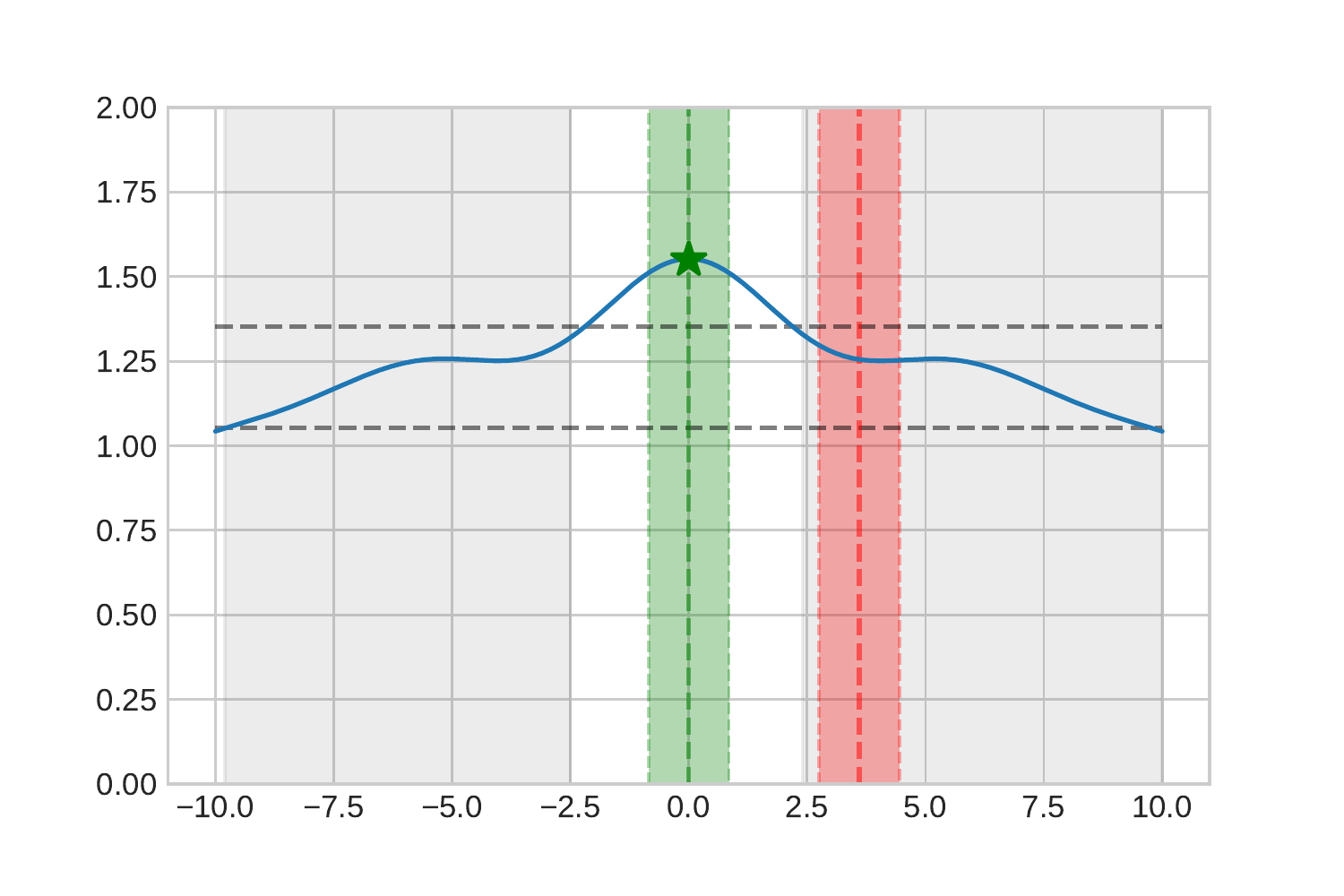}
                \caption{The given function $f$, with its optimum at $x^*$ denoted by the green $\star$.}
            \end{subfigure}%
            ~ 
            \begin{subfigure}[t]{0.5\textwidth}
                \centering
                \includegraphics[height=2.5in, width=\textwidth]{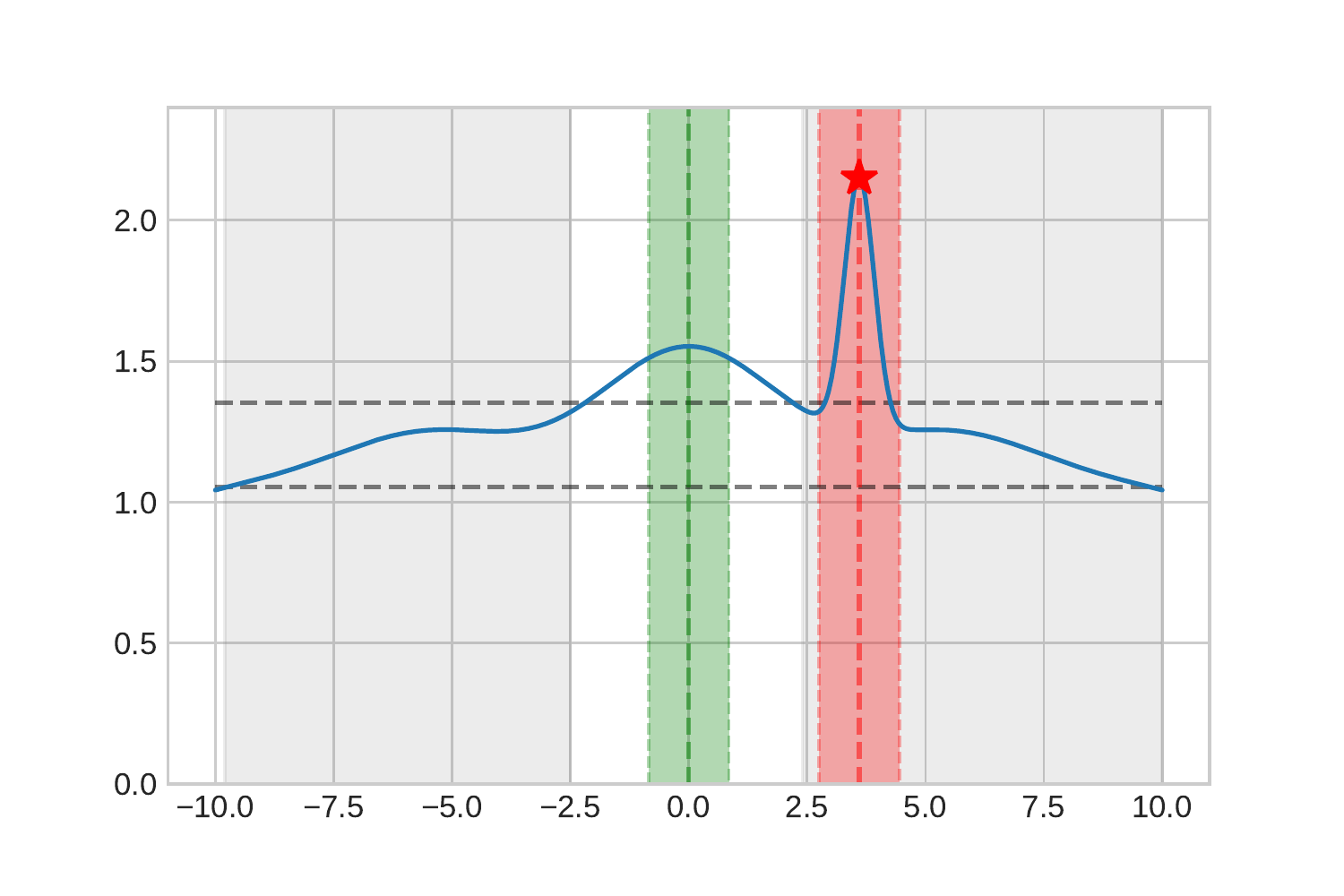}
                \caption{The function $\tilde{f}$ obtained by adding a bump function centered at $z \approx 3.7$ to the original function $f$.}
            \end{subfigure}
            \caption{The gray shaded regions represent the regions $\widetilde{\X}(f, \Delta, c) = \{x: \Delta \leq f(x) - f(x^*) \leq c\Delta\}$ for $\Delta=1.5$ and $c=2$. The green region is $B(x^*, w_0)$ and the red shaded region is $B(z, w)$. By construction, in the left figure, the red-shaded region is at least $\Delta-$suboptimal, while in the right figure the green-shaded region is at-least $\Delta-$suboptimal. The two functions $f$, and $\tilde{f}$ will induce statistically similar distributions as they differ only in a small region (shown in red). Hence, if $f$ is the true objective function,  any good algorithm $\mc{A}$ must spend some queries in the red region to discard the possibility that the objective function is $\tilde{f}$.}
            \label{fig:lower_bound}
        \end{figure*}
        
        The rest of this section is organized as follows: 
        \begin{itemize}
            \item In \Cref{subsubsec:perturbation_general}, we describe the details of the argument in deriving a lower bound on the expected number of samples $\mc{A}$ spends in a suboptimal region for one perturbed function. 
            
            \item In \Cref{subsubsec:perturbation_specific}, we present the details involved in constructing an appropriate collection of perturbed functions. In particular, this involves carefully balancing several trade-offs in the choice of parameters such as $\Delta_i$ and $E_i$, such that the resulting lower bound is tightest. 
        \end{itemize}

    \subsubsection{A perturbation argument}
    \label{subsubsec:perturbation_general}
        Suppose $f \in \rkhs[\kmat](M)$ and $\mc{A}$ is an $a_0$-consistent algorithm. Let $\tilde{f}$ be another function in $\rkhs[\kmat](M)$ that is an $(E, c, \Delta)-$perturbation of $f$, as defined below. 
        \begin{definition}[$(E, c, \Delta)-$perturbation]
        \label{def:perturbation}
            We say that a function $\tilde{f} \in \rkhs[\kmat](M)$ is an $(E, c, \Delta)$-perturbation of another function $f \in \rkhs[\kmat](M)$, if it satisfies the following properties:
            \begin{itemize}
                \item[\tbf{(P1)}] $\tilde{f}$ differs from $f$ only in a region $E$ of the input space, i.e., $\tilde{f}(x) = f(x) $ for all $x \in \X \setminus E$. 
                \item[\tbf{(P2)}] The function $f$ achieves its maximum value (denoted by $f^*$) at a point $x^* \in \X \setminus E$.  On the other hand,  $\tilde{f}$ achieves its maximum value $\tilde{f}^*$ at a point $\tilde{x}^*$ lying in the region $E$.
                
                \item[\tbf{(P3)}] There exist constants $c>1$ and $\Delta>0$ such that the following conditions are satisfied: 
                \begin{align}
                     & |f(x) - \tilde{f}(x)| \leq c \Delta, \quad \text{for all}\; x \in E,\label{eq:perturbation-1}  \\
                     & f^* - f(x) \geq \Delta, \quad  \text{for all}\; x \in E. \label{eq:perturbation-2} \\
                     & \tilde{f}^* - \tilde{f}(x) \geq \Delta, \quad  \text{for all}\; x \in \X \setminus E. \label{eq:perturbation-3}
                \end{align}
            \end{itemize}
        \end{definition}
        
        \begin{remark}
        \label{remark:perturbation-1}
            In this section, we do not address the issue of existence of a function $\tilde{f}$ satisfying all the properties, as well as the possible values of $c$, $\Delta$ and choices of the region $E$. We present the argument under the assumption that such an $\tilde{f}$ exits for some fixed $c$, $\Delta$ and $E$. The trade-offs involved in constructing such $\tilde{f}$ will be discussed in \Cref{subsubsec:perturbation_specific}. 
        \end{remark}
       
        Similar to the more general case of~\ref{eq:regret_decomposition-1}, the  definition above motivates a simple decomposition of the regret in terms of the number of queries made by $\mc{A}$ in the region $E$ in which $f$ and $\tilde{f}$ differ. In particular, if $N$ denotes the (random) number of times $\mc{A}$ queries points in $E$, then we immediately have the following: 
        \begin{align}
            &\mbb{E}_f \lb \mc{R}_n \lp \mc{A}, f \rp \rb \geq \mbb{E}_f \lb N \rb \Delta, \qquad{and} \qquad
            \mbb{E}_{\tilde{f}} \lb \mc{R}_n ( \mc{A}, \tilde{f} ) \rb  \geq \mbb{E}_{\tilde{f}} \lb n - N \rb \Delta. 
        \end{align}
        To obtain a lower bound on $\mbb{E}_f[N]$, we will use the following two key properties of the pair $(f, \tilde{f})$ as encoded by the formal statements in~\Cref{def:perturbation}.
        \begin{itemize}
            \item From a statistical point of view, the two problem instances are close. In particular, $f$ and $\tilde{f}$ only differ over the region $E$, and furthermore their deviation is upper bounded by $c \Delta$. This allows us to upper bound the KL divergence between their induced probability distributions~$\mbb{P}_f$ and $\mbb{P}_{\tilde{f}}$~(see Definition~\ref{def:induced_probability}) in terms of $\mbb{E}_f [N]$.
            
            \item In an operational sense, the two problem instances $\tilde{f}$ and $f$ are sufficiently distinct. This is a consequence of properties~(P2) and~(P3) in \Cref{def:perturbation}, which say that the optimizer of $f$~(resp. $\tilde{f}$) lies in the region $\X \setminus E$~(resp. E) that is known to be at least $\Delta-$suboptimal for $\tilde{f}$~(resp. $f$). 
            This, along with the $a_0$-consistency of $\mc{A}$ will be used to lower-bound the KL-divergence between $\mbb{P}_f$ and $\mbb{P}_{\tilde{f}}$  by a constant. 
        \end{itemize}
        Combining the two inequalities will give us the required lower bound on $\mbb{E}_f[N]$, and consequently on $\mbb{E}_f [ \mc{R}_n \lp \mc{A}, f \rp ]$.
        We now describe the steps.

        \tbf{Step~1: Upper bound on $\bm{D_{KL}(\mbb{P}_f, \mbb{P}_{\tilde{f}})}$.} Recall that the pairs $(f, \mc{A})$ and $(\tilde{f}, \mc{A})$ both induce a probability measure on the $n-$fold product of input-observation space $\Omega \defined (\mc{X} \times \mbb{R} )^{n}$. Denote the two probability measures by $\mbb{P}_f$ and $\mbb{P}_{\tilde{f}}$. Then, assuming that the observation noise is i.i.d. $N(0, \sigma^2)$,  it can be shown that 
        \begin{align}
        \label{eq:overview-2}
            D_{KL} \lp \mbb{P}_f, \mbb{P}_{\tilde{f}} \rp \leq \frac{1}{2\sigma^2}\lp \sup_{x \in E} \tilde{f}(x) - f(x) \rp^2 \mbb{E}_f \lb N \rb \leq \frac{ c^2 \Delta^2}{2 \sigma^2} \mbb{E}_f \lb N \rb. 
        \end{align}
        This intuitive statement states that the KL-divergence between $\mbb{P}_f$ and $\mbb{P}_{\tilde{f}}$ can be controlled by two terms: \tbf{(i)} the maximum deviation between $\tilde{f}$ and $f$, a quantity that is bounded by $c\Delta$ by assumption, and \tbf{(ii)} the expected number of queries made by $\mc{A}$ under $f$. This quantity cannot be too large either, as the algorithm $\mc{A}$ is assumed to be $a_0-$consistent, and the region $E$ is at least $\Delta-$suboptimal for $f$ as stated in~\eqref{eq:perturbation-1}.  
        
        \tbf{Step~2: Lower Bound on $\bm{D_{KL}(\mathbb{P}_{f}, \mathbb{P}_{\tilde{f}})}.$ } Since the algorithm $\mc{A}$ is assumed to be $a_0-$consistent, we immediately have the following two statements for any $a>a_0$:
        \begin{align}
            &\mbb{E}_f \lb N \Delta  \rb  \leq \mbb{E}_f \lb \mc{R}_n \lp f, \mc{A} \rp  \rb = o\lp n^{a} \rp, \quad \text{and} \quad  \mbb{E}_{\tilde{f}} \lb \lp n - N \rp  \Delta \rb \leq  \mbb{E}_{\tilde{f}} \lb \mc{R}_n \lp \tilde{f}, \mc{A} \rp \rb = o \lp n^{a} \rp. 
        \end{align}
        Together, these two conditions imply that $\mbb{E}_f[N]$ and $\mbb{E}_{\tilde{f}}[n-N]$ cannot be too large. To make this formal, define a $[0,1]$ valued random variable $Z = N/n$ and let $p = \mbb{E}_f [Z]$ and $q = \mbb{E}_{\tilde{f}}[Z]$. Note that $Z$ denotes the fraction of samples spent in the region $E$ by the algorithm $\mc{A}$.  Then, for large enough values of $n$, and  $\Delta$ fixed, we expect that $p \approx 0$ and $q \approx 1$. This in turn  implies that the KL-divergence between two Bernoulli random variables with expected values $p$ and $q$ respectively is non-zero. More specifically, 
        we can show that there exists a constant $ C>0$ such that 
        \begin{align}
        \label{eq:overview-3}
          0<  C \leq d_{KL}(p, q) = d_{KL} \lp \mbb{E}_{f} \lb \frac{N}{n} \rb, \, \mbb{E}_{\tilde{f}} \lb  \frac{N}{n} \rb \rp, 
        \end{align}
        where $d_{KL}(p,q) = p \log (p/q) + (1-p) \log \lp (1-p)/(1-q) \rp$ denotes the KL-divergence between two Bernoulli random variables.  The next step in obtaining a lower bound on $\mbb{E}_f [N]$ is the observation that 
        \begin{align}
        \label{eq:overview-4}
        d_{KL} \lp \mbb{E}_{f} \lb \frac{N}{n} \rb, \, \mbb{E}_{\tilde{f}} \lb  \frac{N}{n} \rb \rp \leq D_{KL} \lp \mbb{P}_f, \, \mbb{P}_{\tilde{f}} \rp.  
        \end{align}
        This is a consequence  of data-processing inequality, as shown by \citet{garivier2019explore}.

        \tbf{Step~3: Lower bound $\bm{\mbb{E}_f [N]}$.} Finally, we can use~\eqref{eq:overview-4} to link the statements of~\eqref{eq:overview-2} and~\eqref{eq:overview-3}, and obtain the  inequality 
        \begin{align}
            C < \frac{c^2 \Delta^2}{2 \sigma^2} \mbb{E}_f \lb N \rb, \quad \text{which implies} \quad  
            \mbb{E}_f \lb N \rb \geq \frac{2 C \sigma^2}{c^2 \Delta^2}. 
            \label{eq:lower-1}
        \end{align}
        Multiplying this term with $\Delta$ gives us one term in the regret decomposition of~\eqref{eq:regret_decomposition-1}. 
        
    \subsubsection{Constructing an \texorpdfstring{$(E,c, \Delta)$}{(E, c, D)} perturbed function}
    \label{subsubsec:perturbation_specific}
        We now discuss  the details of constructing a function $\tilde{f}$ satisfying the conditions of \Cref{def:perturbation}. Here is the summary for a fixed $a \in (a_0,1)$.
        \begin{itemize}
            \item   The term  $\Delta$ should not be smaller than $n^{-(1-a)}$. 
            \item An appropriate choice of the set $E$ is  a ball $B(z, w)$ for a point $z \in \X \setminus \X^*$ and radius $w>0$. Recall that $X^* = \{x^* \in \X: f(x^*) = f^* \defined \max_{x \in \X} f(x) \}$. 
            \item We define the perturbed function $\tilde{f} = f + \tilde{g}$ where $\tilde{g}(\cdot) = (c+1) g \lp \frac{\cdot - z}{w} \rp$ for some $w>0$. Here $g$ is the bump function introduced in~\Cref{def:bump}. The specific constraints on $w$ are discussed below.
        \end{itemize}
        
        We now discuss these choices in more details. 
        
        \tbf{Choice of $\Delta$.} The term $\Delta$ parametrizes the amount of perturbation between $f$ and $\tilde{f}$.  $\Delta$ should be large enough to ensure that $\tilde{f}$ and $f$ are \emph{distinguishable} from the point of view of the algorithm $\mc{A}$. In particular, fix any $a>a_0$. Then for this value of $a$, $\Delta$ must be larger than $n^{-(1-a)}$ in order to ensure that $\tilde{f}$ and $f$ are sufficiently distinct. This is because, if $\Delta < n^{-(1-a)}$, then the algorithm may spend all $n$ of its samples in the region $E$ under $f$ as well as $\tilde{f}$ without violating the $o(n^a)$ requirement on regret. 
        
        \tbf{Choice of $z$ and $w$.}  The terms $z$ and $w$ must be such that $B(z, w) \subset \widetilde{\X}(f, \Delta, c) \defined \{x \in \X : \Delta \leq f^* -f(x) \leq c \Delta \}$. Thus $z$ and $w$ must be selected to ensure that the ball of radius $w$ around $z$ is fully contained in the \emph{annular} region $\widetilde{\X}(f, \Delta, c)$ in which the sub-optimality of $f$, (i.e., $f^*-f(x)$) is between $\Delta$ and $c\Delta$.

        \tbf{Defining $\tilde{f}$.} Having defined the region $E$, we then construct the perturbed version of $f$, by adding a shifted and scaled bump function to it. In particular, we add $\tilde{g} = (c+1)\Delta g\lp \frac{\cdot - z}{w} \rp$ to $f$. Note that, by assumption,  the RKHS norm of $f$ satisfies $\|f\|_{\rkhs[\kmat]} < M$. Since we require the perturbed function to also lie in the class $\rkhs[\kmat](M)$, a sufficient condition for that is 
        \begin{align}
            &\|\tilde{g}\|_{\rkhs[\kmat]} \leq \frac{(c+1)\Delta}{w^\nu} \|g\|_{\rkhs[\kmat]} = \frac{(c+1)\Delta}{w^\nu}M_\nu \leq M - \|f\|_{\rkhs[\kmat]} \quad 
            \Rightarrow \;  w \geq  \lp \frac{ (c+1)\Delta M_\nu}{M - \|f\|_{\rkhs[\kmat]}} \rp^{1/\nu}. \label{eq:w-1}
        \end{align}

        \begin{remark}
            \label{remark:diameter}
            Note that the expression for $w$ in~\eqref{eq:w-1} implicitly assumes that for this value of $c$ and $\Delta$ the region $\widetilde{\X}(f, \Delta, c)$ is large enough to contain a ball of this~(or larger) radius. Our result, Theorem~\ref{theorem:general_lower}, holds under this assumption. In case this condition is violated, Theorem~\ref{theorem:general_lower} reduces to the trivial lower bound $\mbb{E}_f[\mc{R}_n \lp f, \mc{A} \rp ] \geq 0$. 
        \end{remark}

    \subsection{A general lower bound}
    \label{subsec:general_lower}
        We can now combine the ideas discussed in the previous section to obtain a general instance-dependent lower bound for $a_0$-consistent algorithms. First, we introduce a notion of complexity associated with a function $f \in \rkhs[\kmat](M)$, that will be used to state the main result. 

        \begin{definition}[\ttt{Complexity-Term}]
            \label{def:complexity-term}
            Let $f \in \rkhs[\kmat](M)$  with $\|f\|_{\kmat} = (1-\lambda)M$ for some $\nu>0$, $M>0$ and $\lambda \in (0, 1)$. Fix a $\Delta>0$, and introduce the set $\mc{Z}_k \defined \widetilde{\X}(f, 2^k\Delta, 2) = \{x \in \X: 2^k\Delta \leq f(x^*)-f(x) < 2^{k+1}\Delta \} $. Introduce the radius $w_k = \lp3 \times  2^{k}\Delta M_\nu/(\lambda M) \rp^{1/\nu}$ and let $m_k$ denote the $w_k$-packing number of the set $\mc{Z}_k$. Finally, define the complexity term 
            \begin{align}
                \label{eq:complexity-term} 
                \comp(\Delta, \nu, M, \lambda) \defined \sum_{k \geq 0} \frac{m_k}{2^{k+2}\Delta} > \frac{m_0}{4\Delta}. 
            \end{align}
            In the sequel, we will suppress the $\nu, M$ and $\lambda$ dependence of $\comp$ and simply use the notation $\comp(\Delta)$. 
        \end{definition}

        We  now present an instance-dependent lower bound on the expected cumulative regret of any $a_0$-consistent algorithm $\mc{A}$ in terms of the complexity term introduced above. 
        
        \begin{theorem}
        \label{theorem:general_lower-1} 
            Let $f \in \rkhs[\kmat](M)$ with $\|f\|_{\kmat} = (1-\lambda)M$ for $\nu>0, M>0$ and $\lambda \in (0,1)$. Consider any $a_0$-consistent algorithm over the family of functions $\rkhs[\kmat](M)$, denoted by $\mc{A}$, and fix any $a>a_0$. Then, the expected cumulative regret of $\mc{A}$ on the instance $f$ satisfies
            \begin{align}
                \label{eq:general_lower-1} 
                \mbb{E} \lb \mc{R}_n \lp \mc{A}, f \rp \rb \geq \frac{7 \log 2}{4} \sigma^2 \comp\lp n^{-(1-a)} \rp, 
            \end{align}
            for $n$ large enough~(exact condition in equation \ref{eq:n_large_enough1} in Appendix~\ref{proof:general_lower-1}). 
        \end{theorem}
        
        \begin{remark}
            \label{remark:general-lower-interpretation}
            The lower bound in~\Cref{theorem:general_lower-1} based on the complexity term of~\Cref{eq:complexity-term} has a natural interpretation. For any $k \geq 0$, the set $\mc{Z}_k$ denotes the `annular' region where the suboptimality $f(x^*)-f(x)$ is between $2^k \Delta_n$ and $2^{k+1}\Delta_n$, with $\Delta_n \defined n^{-(1-a)}$. Then, the term $w_k = \lp 2^{k+1}\Delta_n/(\lambda M) \rp^{1/\nu}$ denotes the radius of the smallest ball that can support a scaled bump function~(see~\Cref{def:bump}) that ensures the resulting perturbed version of $f$ satisfies the properties of~\Cref{def:perturbation}. As discussed earlier in~\Cref{subsubsec:perturbation_general}, for any such perturbation of $f$, the algorithm $\mc{A}$ must spend roughly $1/(2^{k+1}\Delta_n)^2$ samples to distinguish between $f$ and its perturbation. The regret incurred in the process is lower bounded by $2^{k}\Delta_n \times \lp 1/(2^{k+1}\Delta_n)^2\rp = 1/(2^{k+2}\Delta_n)$, that is, suboptimality~($ \geq 2^k\Delta_n$) times the number of queries in that region~($\geq 1/(2^{k+1} \Delta_n)^2$). Since $m_k$  disjoint balls of radius $w_k$ can be packed into $\mc{Z}_k$, the expression of the complexity term in~\eqref{eq:complexity-term} follows. 
        \end{remark}
        
        \Cref{theorem:general_lower-1} follows as a consequence of a more general statement, presented and proved in~\Cref{proof:general_lower-1}. The strict inequality in the definition of the complexity term in~\eqref{eq:complexity-term} immediately implies the following weaker,  but more interpretable,  version of the above statement. 
        
        \begin{corollary}
        \label{corollary:general_lower-2}
            Under the same assumptions as~\Cref{theorem:general_lower-1}, let $m_0$ denote the $w_0\defined \lp 2(n^{-(1-a)})/(\lambda M)\rp^{1/\nu}$ packing number of the set $\mc{Z}_0 \defined \{x \in \X: n^{-(1-a)} \leq f(x^*) - f(x) < 2n^{-(1-a)} \}$. Then, for any $a_0$-consistent algorithm $\mc{A}$, we have the following: 
            \begin{align}
                \mbb{E} \lb \mc{R}_n \lp f, \mc{A} \rp \rb = \Omega \lp \sigma^2 m_0\, n^{(1-a)} \rp. 
            \end{align}
        \end{corollary}
        
         \begin{remark}
            \label{remark:near_optimality_dimension}
            Corollary~\ref{corollary:general_lower-2} states that a key quantity characterizing the regret achievable by a uniformly good algorithm is the packing number of an annular near-optimal region associated with the given function. Informally, we can write $m_0 \approx w_0^{-\tilde{d}} \approx n^{(1-a)\tilde{d}/\nu}$, where $\tilde{d} \defined \liminf_{n \to \infty} \log(m_0)/\log(1/w_0)$.  The term $\tilde{d}$ is reminiscent of the concepts of \emph{near-optimality dimension} and \emph{zooming dimension} used in prior works in bandits in metric spaces~\citep{bubeck2011x, kleinberg2019bandits} as well  as in Gaussian Process bandits~\citep{shekhar2018gaussian}.  These works use similar notions to obtain upper bounds on the regret of algorithms that non-uniformly discretize the domain.  
        \end{remark}

        \begin{figure*}[hbt]
            \centering
            \begin{subfigure}[t]{0.5\textwidth}
                \centering
                \includegraphics[height=2.5in, width=\textwidth]{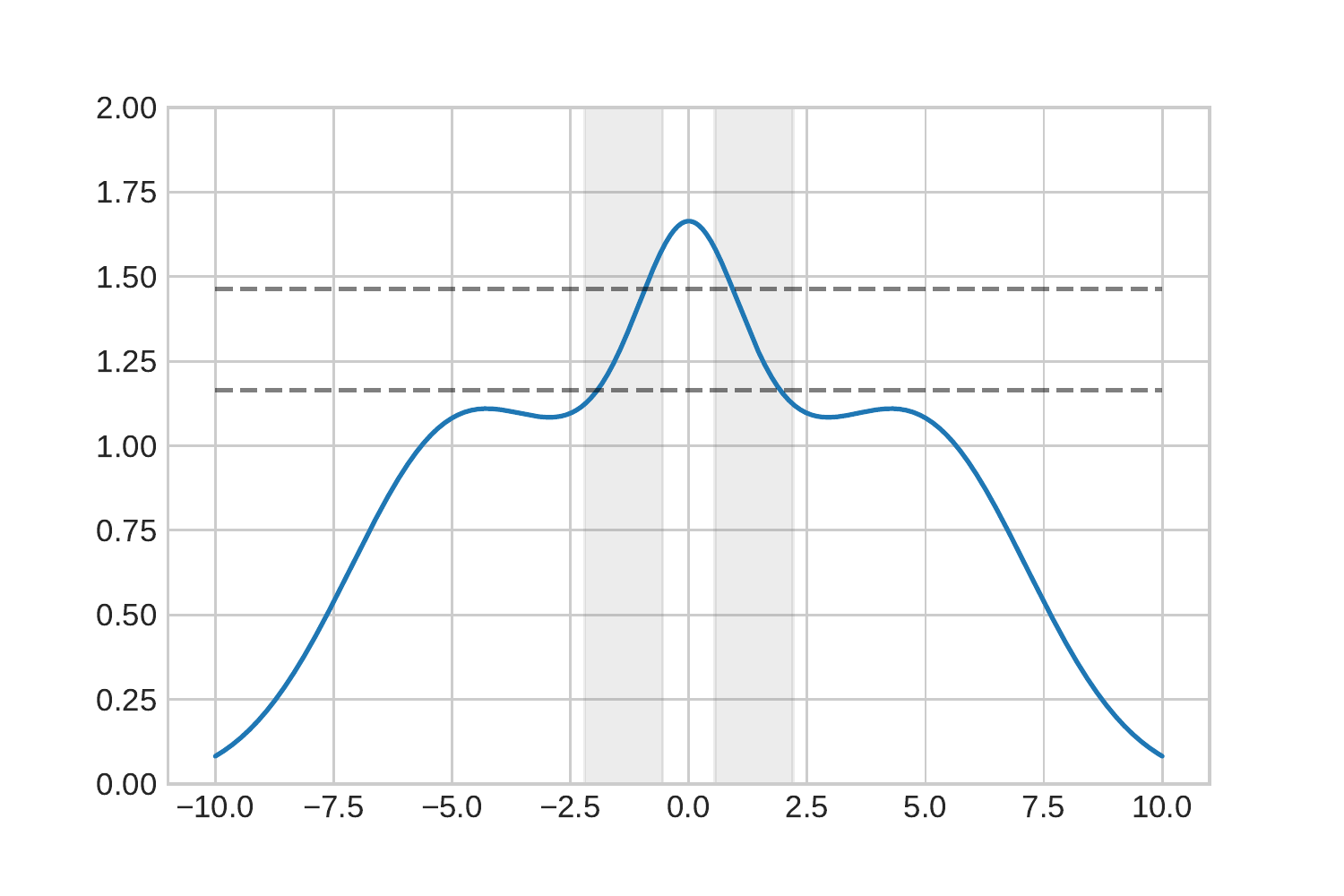}
            \end{subfigure}%
            ~ 
            \begin{subfigure}[t]{0.5\textwidth}
                \centering
                \includegraphics[height=2.5in, width=\textwidth]{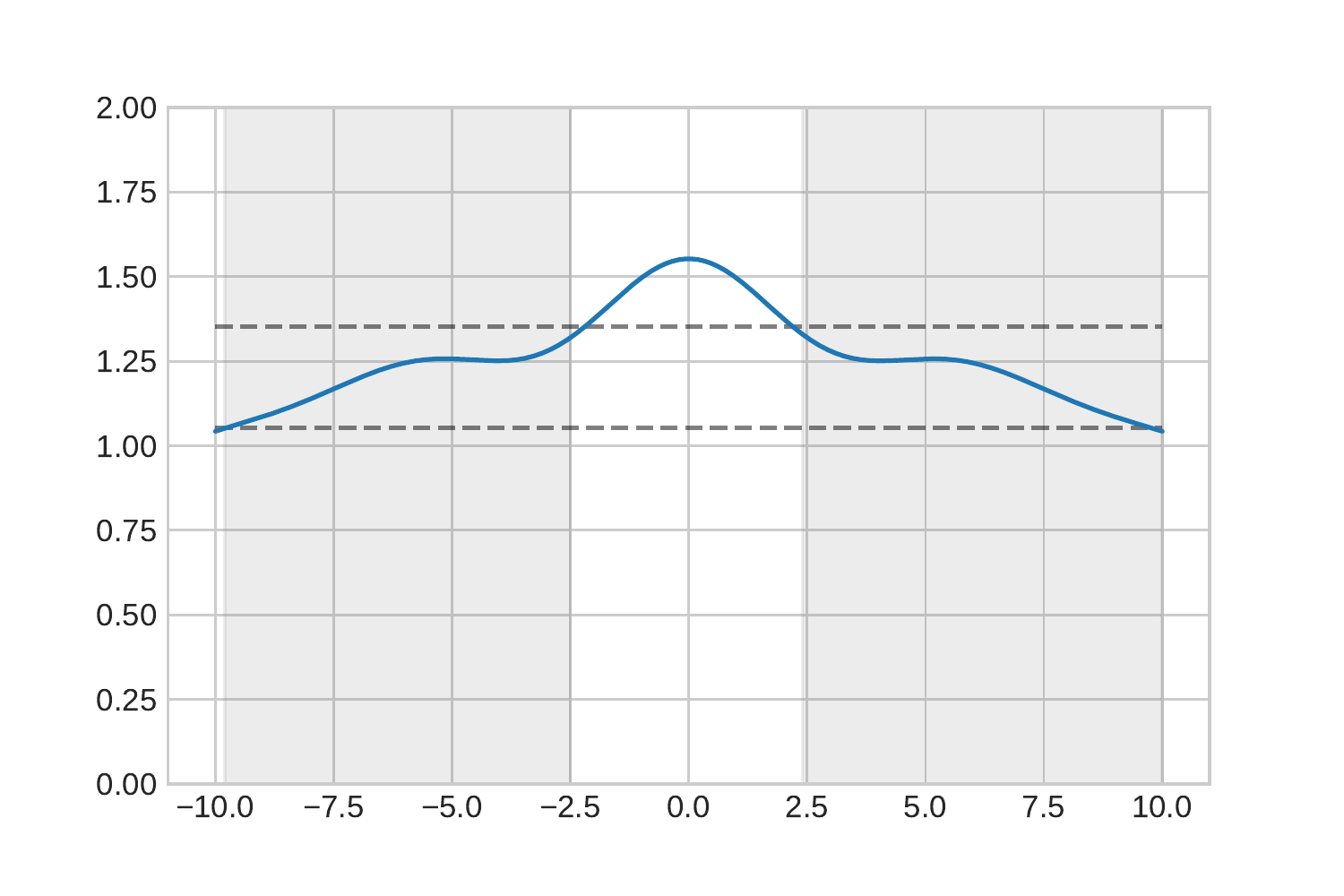}
            \end{subfigure}
            \caption{The gray shaded regions in the two figures depict the annular near-optimal set $\mc{Z}_0$ used in the lower bound in~\Cref{subsec:growth_lower}. As stated in the corollary, the complexity associated with a function can be lower-bounded by the $w_0$-packing number of the set $\mc{Z}_0$. Since the set $\mc{Z}_0$ associated with the figure on the right has a larger volume, it can be packed with more balls of radius $w_0$. Hence, the instance-dependent bound for the figure on the right is larger than the corresponding bound for the figure on the left. }
            \label{fig:toyfigures2}
        \end{figure*}
  
    In the next section, we specialize the above results to functions that satisfy an additional `local growth' condition, for which the complexity terms can be explicitly lower bounded in terms of more interpretable parameters. 
    
\subsection{Lower bound under growth condition}
\label{subsec:growth_lower}
In this section, we consider the class of functions satisfying the following additional assumption.

\begin{customassump}{3}[Growth Condition]
\label{assump:growth} 
We say that the objective function $f$ satisfies the local growth condition with parameters $(\cl, \cu, a, r_0)$ if for all $x \in B(x^*, r_0) \cap \X$, we have $\cl \|x - x^*\|^b \leq f(x*) - f(x) \leq \cu \|x - x^*\|^b$. We shall denote by $\mc{F}\lp \cl, \cu, b, r_0 \rp$ the class of all functions satisfying this property.
\end{customassump}
 Similar conditions have been used in analyzing the performance of first order stochastic optimization algorithms by~\citet{ramdas2013optimal} and in characterizing the minimax rates of active learning algorithms by~\citet{castro2008minimax}. 

As an example, consider the case when the function $f$ has continuous second order derivatives and its optimizer $x^*$ lies in the interior of the domain. Then, if the Hessian of $f$ at $x^*$ is non-singular, then we can find an $r_0>0$ such that for all $x$ in $B(x^*, r_0)$, the spectral norm of the Hessian of $f$ is between $\cl$ and $\cu$. Then the function $f$ satisfies the growth condition with exponent $2$ and constants $\cl$ and $\cu$. We now state the main result of this section. 

\begin{proposition}
\label{theorem:lower_local}
    Introduce the function class $\mc{G} \defined \rkhs[\kmat](M) \cap \mc{F} \lp  \cl, \cu, b, r_0 \rp$ for some $\nu>0$ and $b>\nu$.  Let $\mc{A}$ be an $a_0-$consistent  algorithm $a_0$ for the function class $\rkhs[\kmat](M)$. Then for any $a>a_0$ and $f \in \mc{G}$ with $\|f\|_{\rkhs[\kmat]}<M$, we have the following: 
    \begin{align}
        \liminf_{n \to \infty} \; \frac{ \mbb{E}\lb \mc{R}_n\lp \mc{A}, f \rp \rb }{ n^\alpha   } > 0 \qquad \text{for any } \quad \alpha < \lp 1 - a_0 \rp \lp 1 + \frac{d}{\nu} \lp 1 - \frac{\nu}{b} \rp \rp. 
    \end{align}
\end{proposition}

Since the existing algorithms, discussed in \Cref{subsec:related_work_bead}, satisfy the uniform regret condition introduced in Definition~\ref{def:uniform}, we can use Theorem~\ref{theorem:lower_local} to obtain the instance-dependent lower bounds for these algorithms. 

\begin{corollary}
\label{corollary:kernelucb}
    With $\mc{A}$ set to the  \kernelucb algorithm, Theorem~\ref{theorem:lower_local} implies the following for any $\nu>d/2$ and $f \in \mc{G}$: 
    \begin{align}
        \label{eq:lower_kernelucb}
            \liminf_{n \to \infty} \; \frac{ \mbb{E}\lb \mc{R}_n\lp \mc{A}, f \rp \rb }{ n^\alpha   } > 0 \qquad \text{for any } \quad \alpha <\lp 1 - \frac{d}{2\nu} \rp \frac{\nu  + d(1-\nu/b)}{2\nu + d}. 
    \end{align}
\end{corollary}

\begin{corollary}
\label{corollary:gpucb}
With $\mc{A}$ set to either \gpucb or \gpts, Theorem~\ref{theorem:lower_local} implies the following for any $\epsilon>0$ and $f \in \mc{G}$: 
\begin{align}
    \label{eq:lower_gpucb}
         \liminf_{n \to \infty} \; \frac{ \mbb{E}\lb \mc{R}_n\lp \mc{A}, f \rp \rb }{ n^\alpha   } > 0 \qquad \text{for any } \quad \alpha < \frac{\nu  + d(1-\nu/b)}{2\nu + d}.    
\end{align}
\end{corollary}

\begin{corollary}
\label{corollary:pigpucb}
With $\mc{A}$ set to the  \pigpucb algorithm, Theorem~\ref{theorem:lower_local} implies the following for any $\nu>0$ and $f \in \mc{G}$: 
\begin{align}
    \label{eq:lower_pigpucb}
        \liminf_{n \to \infty} \; \frac{ \mbb{E}\lb \mc{R}_n\lp \mc{A}, f \rp \rb }{ n^\alpha   } > 0 \qquad \text{for any } \quad \alpha <  \frac{ (d/\nu + 2) \lp \nu + d(1- \nu/b) \rp }{ d(2d + 4) + 4\nu }. 
\end{align}
\end{corollary}

The results of the above corollaries are presented for some specific $\nu$ and $b$ values in Figure~\ref{fig:lower_bound}. As we can see, algorithms with tighter uniform regret bounds incur higher instance-dependent lower bounds. Furthermore, the gap between the upper and lower bound decreases with increasing $b$. 

\begin{figure*}[t!]
    \centering
    \begin{subfigure}[t]{0.5\textwidth}
        \centering
        \includegraphics[height=2.5in, width=\textwidth]{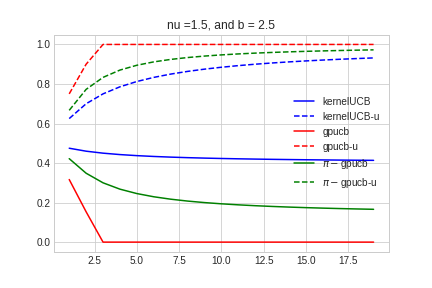}
        \caption{}
    \end{subfigure}%
    ~ 
    \begin{subfigure}[t]{0.5\textwidth}
        \centering
        \includegraphics[height=2.5in, width=\textwidth]{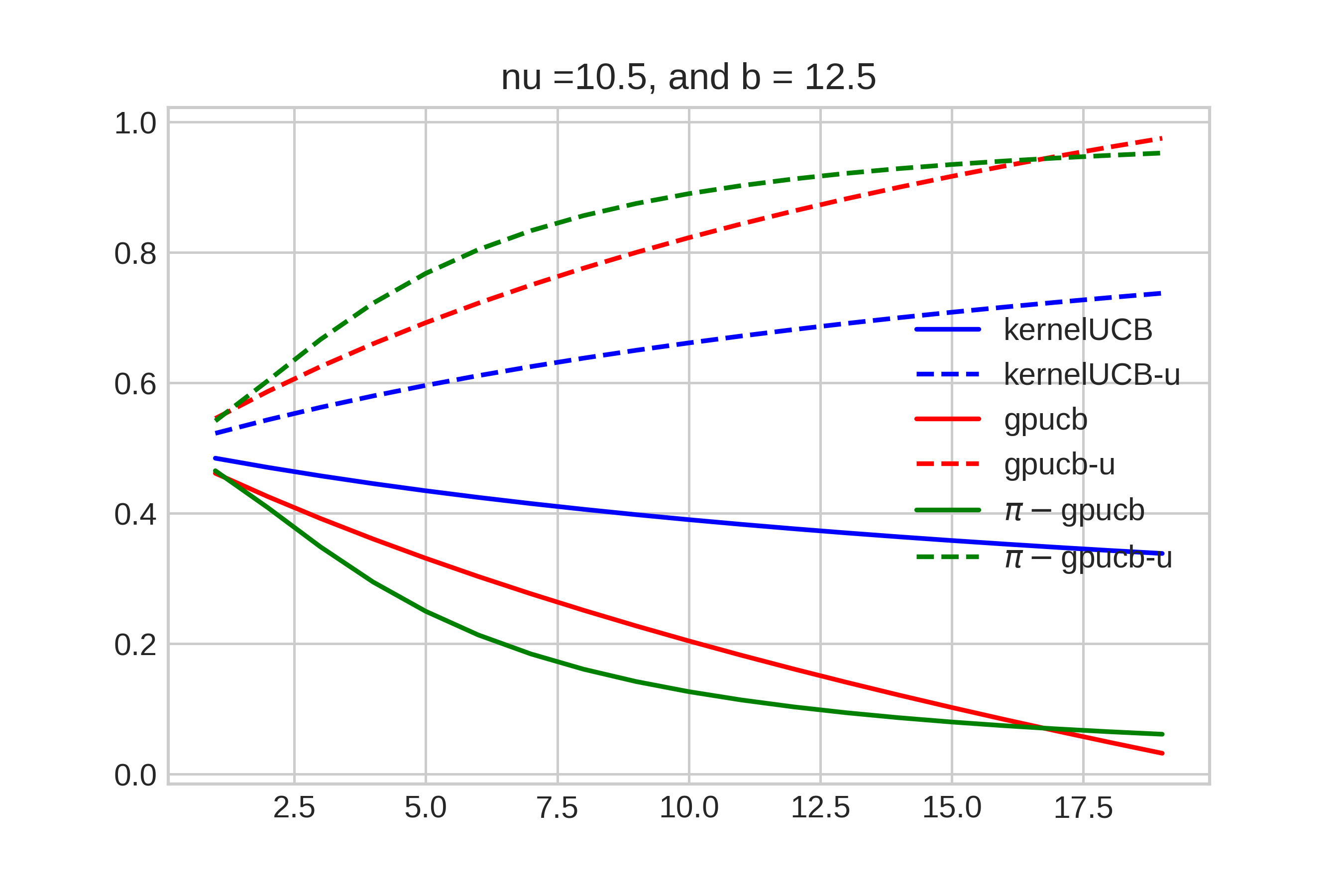}
        \caption{}
    \end{subfigure}
    \caption{The gap between the instance-dependent lower bounds~(solid curves) under growth condition derived in~\Cref{theorem:lower_local}, and the worst-case regret upper bounds~(dashed curves) of different commonly used algorithms for two pairs of $\nu$~(the smoothness parameter) and $b$~(the growth rate parameter) values.}
\end{figure*}

\section{Instance-Dependent Upper Bound}
\label{sec:upper-bound}
As mentioned in the introduction, the theoretical analysis of the existing kernelized bandits algorithms upper bound their regret in terms of quantities such as the maximum information gain, $\igain$, that depend on the entire function class. Hence, such results do not adapt to the hardness of the specific problem instance within the class -- they predict the same upper bound for the `easiest' as well as the `hardest' problem in the class. We take a step towards addressing this issue, and describe a simple algorithm that is minimax near-optimal but also admits tighter upper bounds for easier problem instances.

\begin{definition}[\ttt{Upper-Complexity}]
\label{def:upper-complexity}
Consider a function $f \in \rkhs[\kmat]$ with $\nu>0$, and define $\xi = \min\{1, \nu\}$. For given constants $\Delta, \rho \in (0,1)$ and $c_1, c_2>0$,  introduce the set $\widetilde{\mc{Z}}_k \defined \{x \in \X: f(x^*) -f(x) \leq c_1  (1/\rho)^{k \xi} \Delta \}$ for $k \geq 0$, and let $\widetilde{m}_k$ denote the $2 \widetilde{w}_k \defined 2 c_2\lp (1/\rho)^{k\xi}\Delta\rp^{1/\xi}$ packing number of the set $\widetilde{\mc{Z}}_k$. With these terms introduced,  define the upper-complexity term as follows:
\begin{align}
    \label{eq:upper-complexity}
    \compupper(\Delta, \nu, \rho, c_1, c_2)  \defined   \sum_{k \geq 0}  \frac{\widetilde{m}_k}{ (1/\rho)^{k \xi} \Delta }. 
\end{align}
In the sequel, we will drop the $\nu$, $\rho$, $c_1$ and $c_r2$ dependence of the complexity term, and simple denote it by $\compupper(\Delta)$. 
\end{definition}

\begin{remark}[Comparison of $\comp$ and $\compupper$]
\label{remark:comparison-of-complexity-terms}
    The upper-complexity term introduced above has a similar form as the corresponding lower-complexity term~$(\comp)$, introduced earlier in~\Cref{def:complexity-term}: both complexity measures sum over terms involving a packing number of a near-optimal set in the numerator, and an exponentially growing term times $\Delta$ in the denominator.
    Despite this similarity, the upper-complexity term is, in general, larger than the corresponding lower-complexity term. This is because $\widetilde{w}_k$ is proportional to $\Delta^{1/\xi}$, while $w_k$~(in~\Cref{def:complexity-term}) is proportional to $\Delta^{1/\nu}$. Since $\nu \geq \xi \defined \min \{1, \nu\}$, and $\Delta<1$, the term $\widetilde{m}_k$ represents a much tighter packing than the corresponding term, $m_k$ in~\Cref{def:complexity-term}. Another, less important, factor in $\compupper$ being larger than $\comp$ is that the set $\widetilde{\mc{Z}}_k$ is usually larger than the corresponding `annular' set $\mc{Z}_k$ used in defining $\comp$. 
\end{remark}

We now state the main result of this section stating that there exists a minimax near-optimal algorithm, whose instance-dependent regret can be characterized by the upper-complexity term defined above. The details of the algorithms are presented in~\Cref{sec:algorithms_bead}. 

\begin{theorem}
\label{theorem:general-upper}
    For the class of functions $f \in \rkhs[\kmat](M)$, there exists an algorithm~(denoted by $\mc{A}_1$) that is $a_0$-consistent with $a_0 a_{\nu}^* =(\nu+d)/(\nu+2d)$, and also satisfies the following instance-dependent upper bound on the expected regret for $n$ large enough: 
    \begin{align}
        \label{eq:general-upper-0}
        \mbb{E}\lb \mc{R}_n \lp \mc{A}_1, f \rp \rb = \widetilde{\mc{O}} \lp \compupper \lp \Delta_n \rp \rp,   
    \end{align}
    where $\Delta_n = \min \{n^{-(1-a_{\nu}^*)}, \rho^{H_n \xi}\}$, where $H_n$ is defined precisely in~\Cref{lemma:regret_bound3} in~\Cref{appendix:bead_upper}.  The $\widetilde{\mc{O}}$ term suppresses polylogarithmic factors in $n$. Note that the value of $a_0$ stated above implies that the algorithm $\mc{A}_1$ is minimax near-optimal. 
\end{theorem}

We now specialize the above result to the special case in which the functions also satisfy the additional local-growth condition. 

\begin{proposition}
\label{theorem:bead_upper}
Suppose the Assumption~\ref{assump:rkhs} holds with $K = \kmat$ for some $\nu>0$, Assumption~\ref{assump:noise-upper} holds with parameter $\sigma^2$,  and Assumption~\ref{assump:growth} holds with exponent $b$. Then, the cumulative regret of Algorithm~\ref{algo:bead} satisfies the following, with $\xi = \min\{1, \nu\}$: 
\begin{align}
    \label{eq:bead_regret}
    \mbb{E} \lb \mc{R}_n \lp \mc{A}_1, f \rp \rb = \widetilde{\mc{O}}(n^a), \quad \text{where }\;  a \defined   \min \lp  \frac{d+\nu}{d+2\nu}, \; \frac{ d(1-\xi/b)^{+} + \xi}{d(1-\xi/b)^{+} + 2\xi} \rp.
\end{align}
The notation $(z)^+$ refers to $\max \{0, z\}$ and the notation $\tilde{\mc{O}}$ hides the polylogarithmic factors in the upper bound. 
\end{proposition}

The proof of this result is given in Appendix~\ref{appendix:bead_upper}. In particular, this result implies that for a fixed $b>\nu$, for all values of $d \geq 1$ and $0 < \nu < \frac{1}{1-b}$, the regret achieved by Algorithm~\ref{algo:bead} is tighter than the minimax rate (achieved by \kernelucb). Furthermore, for a fixed $\nu>0$, the amount of possible improvement increases with decreasing values of $b$.  
\begin{figure*}[t!]
    \centering
    \begin{subfigure}[t]{0.5\textwidth}
        \centering
        \includegraphics[height=2.5in, width=\textwidth]{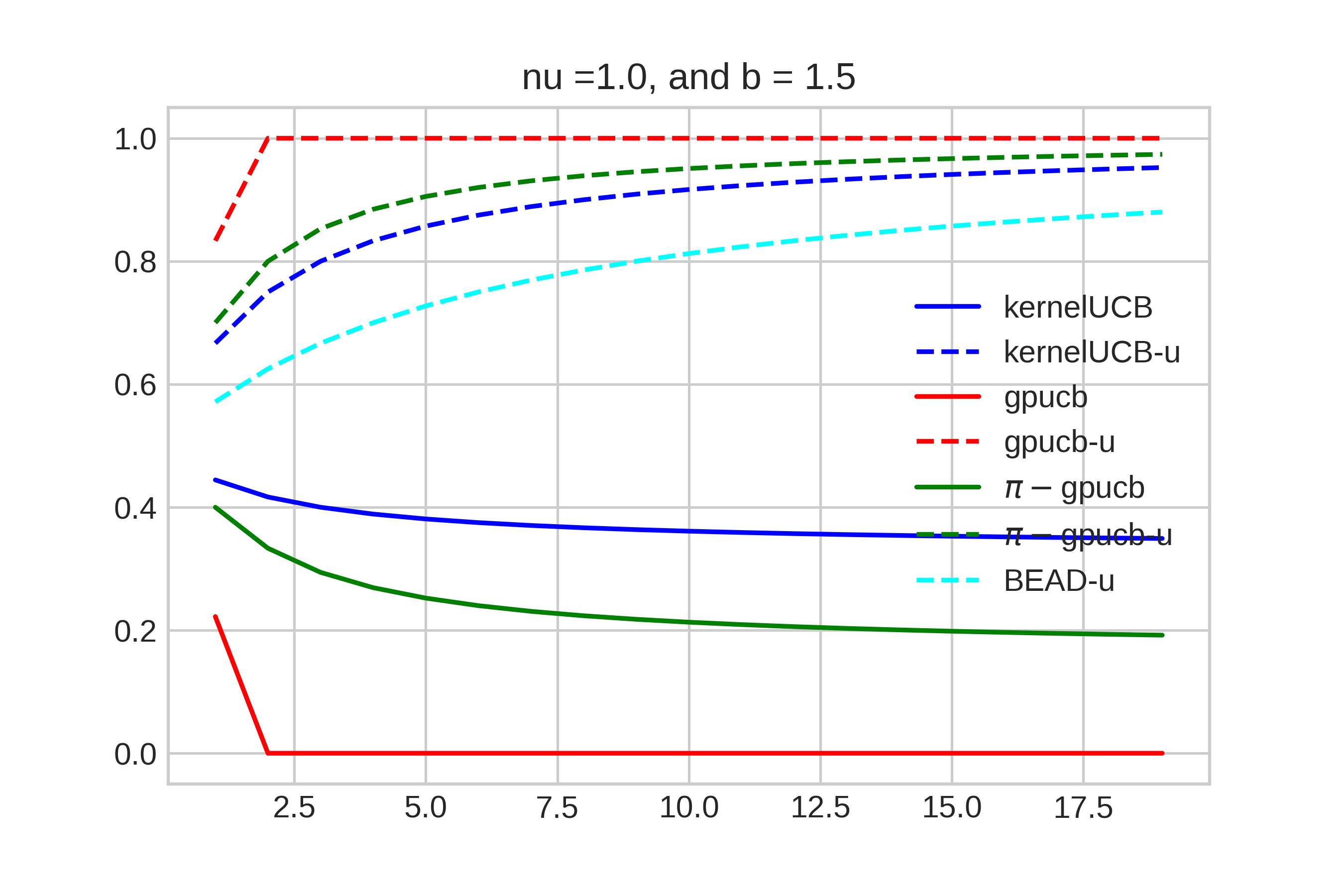}
    \end{subfigure}%
    ~ 
    \begin{subfigure}[t]{0.5\textwidth}
        \centering
        \includegraphics[height=2.5in, width=\textwidth]{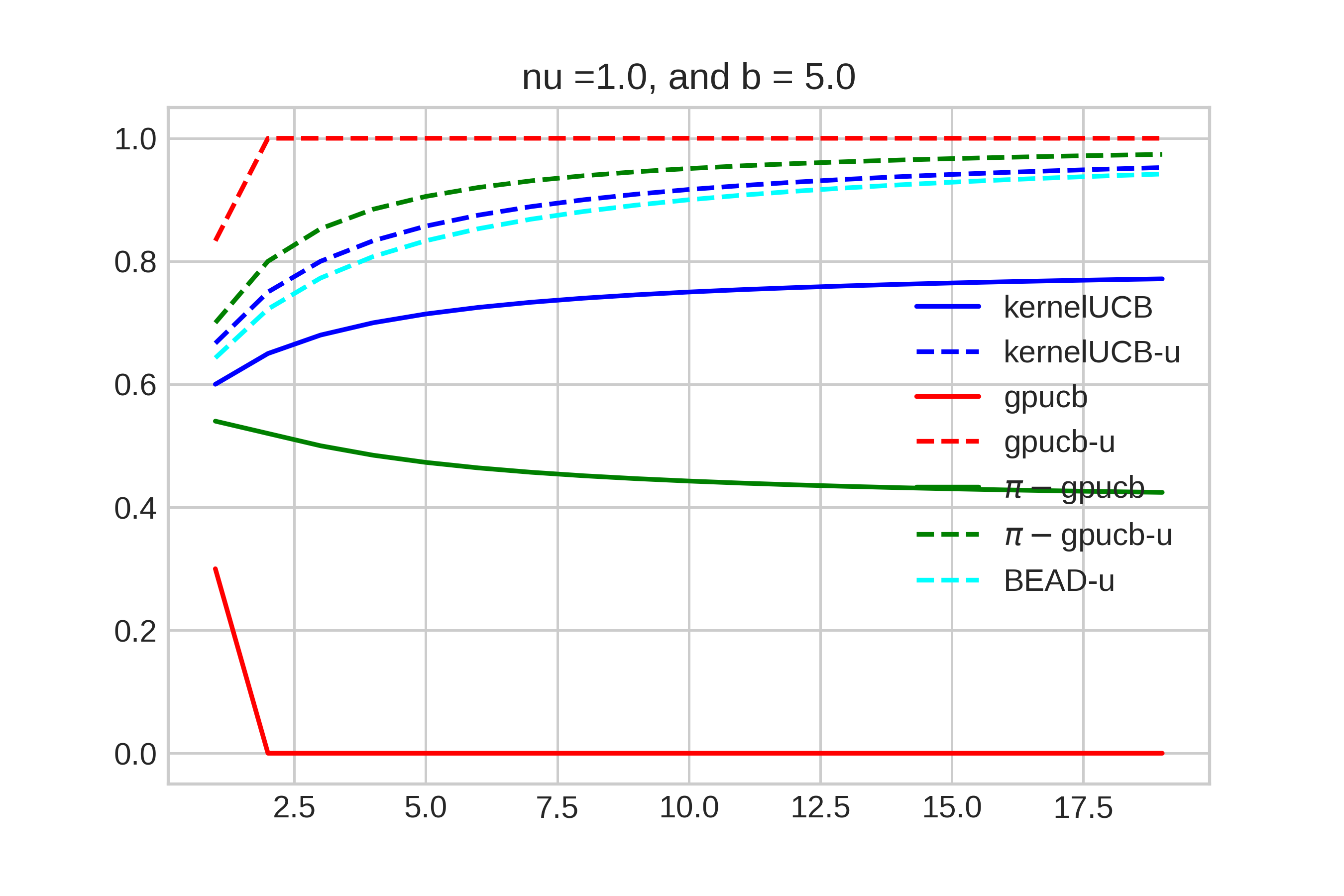}
    \end{subfigure}
    \caption{The light-blue dashed lines in the above figures show the regret bound for Algorithm~\ref{algo:bead} derived in Theorem~\ref{theorem:bead_upper} for two different values of $b$. In both these cases, the regret of Algorithm~\ref{algo:bead} goes beyond the minimax rate (dashed blue line). Furthermore, the amount of improvement over the minimax rate increases with decreasing $b$, i.e., as the problem instance becomes easier.}
    \label{fig:bead}
\end{figure*}

\subsection{Proposed Algorithm}
\label{sec:algorithms_bead}


\begin{algorithm}[htb]
\SetAlgoLined
\SetKwInOut{Input}{Input}\SetKwInOut{Output}{Output}
\SetKwFunction{LocalPoly}{LocalPoly}
\SetKwFunction{Expand}{Expand}
\newcommand\mycommfont[1]{\ttfamily\textcolor{blue}{#1}}
\SetCommentSty{mycommfont}

\KwIn{$n$, the querying budget \\
$\kmat$,  the kernel belonging to the \matern family\\
$M$, upper bound on the RKHS norm \\
$v_1, v_2, \rho_0$, parameters of the tree of partitions \\
$\tau$, regularization parameter used in posterior computation. \\
\BlankLine
}
 \tbf{Initialize:}~ $t=1$,  
 $\texttt{flag} \leftarrow True$, $\mc{E}_t \leftarrow \emptyset$, $\mc{Y}_t \leftarrow \emptyset$, $\mc{P}_t \leftarrow \{x_{0,1}\}$, $L=M\kmat(0)$ and $\xi = \min\{\nu, 1\}$. \\
 \For{$t=1,2,\ldots, n$}{
 \While{$\texttt{flag}$}{
 $\beta_t, \{(\mu_t(x), \sigma_t(x)): x \in \mc{P}_t \} = \posterior\lp \mc{P}_t, \mc{E}_t, \mc{Y}_t, \tau \rp $ \\
$ x_t \in \argmax_{x \in \mc{P}_t}\;   \sigma_t(x), \qquad $
$ U_t = \max_{x \in \mc{P}_t}\;   \sigma_t(x) $ \\
\BlankLine
\uIf{$1/\sqrt{n} < U_t < L \lp v_1 \rho^{h_t} \rp^{\xi}$} 
{  
$\mc{P}_t, h_t \leftarrow \refine(\mc{P}_t, \beta_t,  \{(\mu_t(x), \sigma_t(x)): x \in \mc{P}_t\})$ \\

$\mc{E}_t \leftarrow \emptyset$, $\quad \mc{Y}_t \leftarrow \emptyset$
} 
\Else{
$y_t \leftarrow f(x_t) + \eta_t, \quad$ 
$\mc{E}_t \leftarrow \mc{E}_t \cup \{x_t\}, \quad$
$\mc{Y}_t \leftarrow \mc{Y}_t \cup \{y_t\}$\;
$\ttt{flag} \leftarrow False$
} 
 } 
 
$\ttt{flag} \leftarrow True$\;
} 
 \caption{Breadthwise Exploration with Adaptive Discretization~($\mc{A}_1$)}
 \label{algo:bead}
\end{algorithm}


We first introduce a standard notion of a sequence of nested partitions of the input space, often used in prior works, such as~\citep{bubeck2011x, munos2011optimistic, wang2014bayesian, shekhar2018gaussian}, to design algorithms for zeroth order optimization. 

\begin{definition}[tree of partitions] 
\label{def:tree} 
We say that a sequence of subsets of $\X$, denoted by $\lp \X_h \rp_{h\geq 0}$, forms a tree of partitions of $\X$, if it satisfies the following properties: 
\begin{itemize}\itemsep0em 
\item For all $h \geq 0$, we  have $\X_h = \{x_{h,i} : 1 \leq i\leq 2^i\}$. Furthermore, for every $x_{h,i} \in\X_h$ is associated a cell $\X_{h,i}$. For $i \neq j$, $\X_{h,i}$ and $\X_{h,j}$ are disjoint, and $\cup_{i=1}^{2^h}\X_{h,i} = \X$. Moreover, for any $h,i$ we have $\{ x_{h+1, 2i-1}, x_{h+1, 2i} \} \in \X_{h+1} \cap \X_{h,i}$. 

\item There exist constants $0<v_1 \leq 1 \leq v_2$ and $\rho \in(0,1)$ such that 
\begin{align}
    B(x_{h,i}, v_2 \rho^h) \subset \X_{h,i} \subset B(x_{h,i}, v_1\rho^h). 
\end{align}
\end{itemize}
\end{definition}

Next, we introduce two subroutines that will be employed by the algorithm. The first subroutine, called~\posterior, computes the posterior mean and posterior covariance function of the Gaussian Process~(GP) model used for approximating the unknown objective function in the algorithm. 

\begin{definition}[\posterior]
\label{def:compute_posterior}
Given a subset $\mc{P}_t \subset \X$, a multi-set of points belonging to $\mc{P}_t$, denoted by $\mc{E}_t = \{x_1, \ldots, x_{t-1}\}$, at which the function $f$ was evaluated,  the corresponding noisy function evaluations $\mc{Y}_t=\{y_1, \ldots, y_{t-1}\}$ and a constant $\tau>0$, the \posterior subroutine returns the terms $\beta_t$, and $\{(\mu_t(x), \sigma_t(x)) : x \in \mc{P}_t \}$, which are defined as follows: 
\begin{align}
    & \Sigma_t = K_t + \tau I, \qquad \text{where} \quad K_t = [K(x_i, x_j)]_{x_i,x_j \in \mc{E}_t} \\
    & k_t(x) = [K(x, x_1), \ldots, K(x, x_{t-1})]^{T}, \qquad \bm{y_t} = [y_1, y_2, \ldots, y_{t-1}]^{T}\\
    &\beta_t = \sqrt{ 2 \log \lp |\mc{P}_t| n^3 \rp / t} \\
    & \mu_t(x) = k_t(x)^{T} \Sigma_t^{-1} \bm{y_t}   \\
    & \sigma_t(x) = \tau^{-1/2} \sqrt{ K(x,x) - k_t(x)^{T} \Sigma_t^{-1} k_t(x)}. 
\end{align}
\end{definition}

As we describe later, our algorithm proceeds by adaptively discretizing the input space based on the observations gathered -- the granularity of the partition becoming finer in the near-optimal regions of the input space, aimed at mimicking the discretization of the space involved in defining the complexity term in~\Cref{def:complexity-term}. Our second subroutine, called~\refine, presents the formal steps involved in updating the discretization used in the algorithm. 

\begin{definition}[\refine]
\label{def:refine_partition}
\sloppy The \refine~subroutine takes in as inputs, $\mc{P}_t, h_t, \beta_t$ and $\{ (\mu_t(x), \sigma_t(x) ) : x \in \mc{P}_t \}$; and returns an updated partition $\mc{P}_t$ and level $h_t$. First compute $l_t \defined \min_{x \in \mc{P}_t} \mu_t(x) - \beta_t \sigma_t(x)$ and define $\tilde{\mc{P}}_t = \{x \in \mc{P}_t : \mu_t(x) + \beta_t \sigma_t(x) > l_t \}$. Next, it defines the new $\mc{P}_t$ as $\mc{P}_t = \cup_{x_{h_t,i} \in \tilde{\mc{P}}_t } \{x_{h_t+1, 2i-1}, \, x_{h_t+1, 2i} \}$, and updates $h_t \leftarrow h_t+1$. 
\end{definition}

We now present an outline of the steps of our proposed algorithm below. The formal pseudocode is in~\Cref{algo:bead}. 

\noindent{\emph{Outline of \algoref{algo:bead}.}} At any time $t \geq 1$, the algorithm maintains a set of \emph{active points} denoted by $\mc{P}_t$. These points satisfy the following two properties: \tbf{(i)} $\mc{P}_t \subset \mc{X}_{h_t}$ for some $h_t \geq 0$; i.e., all the active points lie in the same `depth' (i.e., $h_t$) of the tree of partitions, and \tbf{(ii)} any optimizer $x^*$ of $f$ must lie in the region $\cup_{x_{h_t, i} \in \mc{P}_t} \X_{h_t, i}$. The algorithm evaluates the function at points in the active set, and computes the posterior mean and standard deviation by calling the \posterior subroutine. The algorithm then compares the maximum posterior standard deviation (for points in $\mc{P}_t$) with an upper bound on the variation in function value in the cell $\mc{X}_{h_t, i}$ associated with an active point $x_{h_t, i} \in \mc{P}_t$. If the maximum posterior standard deviation is larger than $L(v_1\rho^h)^\xi$, then the algorithm evaluates the function at the corresponding active point with the largest $\sigma_t(x)$. Otherwise, it concludes that the active points in $\mc{P}_t$ have been sufficiently well explored, and it moves to the next level of the partition tree by calling the \refine subroutine. The above process continues until the querying budget is exhausted.  

\begin{remark}
\label{remark:bead}
As the name suggests, Algorithm~\ref{algo:bead} carefully combines two key ideas from kernelized bandits literature: \tbf{(i)}  it divides the evaluated points into subsets~(according to their level $h$ in the tree) which satisfy a conditional independence property, similar to \kernelucb of \citet{valko2013finite}, and \tbf{(ii)} adaptively partitions the input space to zoom into the near-optimal regions, similar to the algorithms in \citep{shekhar2018gaussian, shekhar2020multi}. The first property allows us to construct tighter confidence intervals, which results in the algorithm achieving the minimax regret rate. As a result, applying Theorem~\ref{theorem:general_lower} to this algorithm provides us with the best (i.e., the highest) instance-dependent lower bounds. Additionally, the second property allows the algorithm to exploit the `easier' problem instances when the objective function satisfies the growth condition with small $b$, and results in improved regret bound in these problem instances. This is shown in Figure~\ref{fig:bead}. 
\end{remark}

\begin{figure}
    \centering
    \includegraphics[width=\textwidth, height=3in]{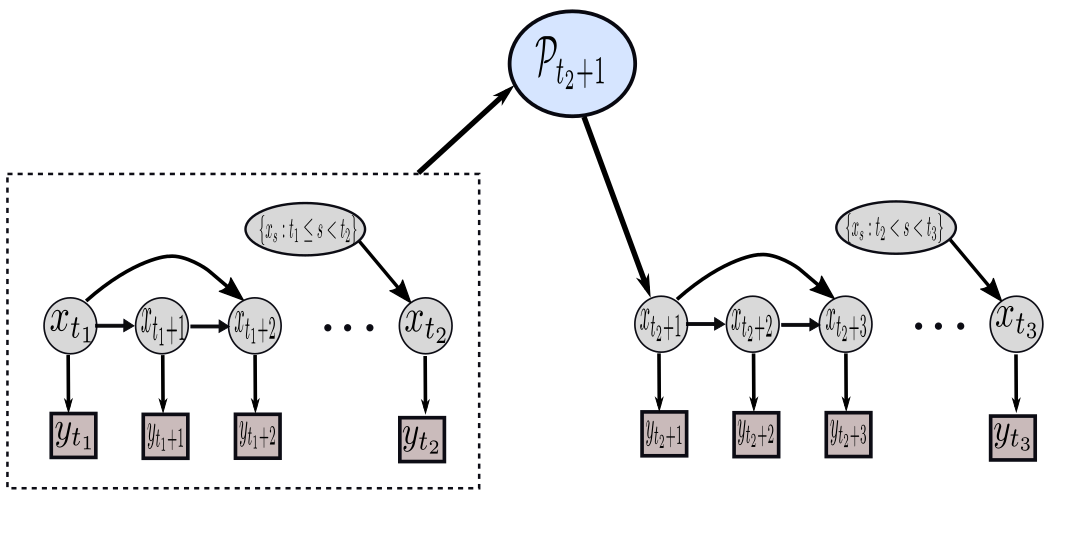}
    \caption[Dependence structure of the query points]{ 
    Suppose Algorithm~\ref{algo:bead} queries points from level $h$ of the tree in the time interval $\{t_1, \ldots, t_2\}$, refines the partition at time $t_2+1$ and then queries points from the level $h+1$ in the time interval $\{t_2+1, \ldots, t_3\}$, for some positive integers $t_1<t_2<t_3$. Then the above diagram shows the dependence structure among the queried points and the observations. Notably, similar to the \kernelucb algorithm, the breadth-wise exploration achieves the conditional independence among the observations $\{y_s: t_1\leq s \leq t_2\}$ given $\{x_s: t_1 \leq s \leq t_2\}$. This is because we only employ the posterior standard-deviation in the exploration within a  level $h$ of the tree, and the fact that conditional variance does not depend on the observations. Finally, we note that the cell refinement rule uses information from all the $(x_s, y_s)$ pairs at level $h$ to get the new active set $\mc{P}_{t_2+1}$. 
    }
    \label{fig:conditional_independence}
\end{figure}

\section{Conclusion}
\label{sec:conclusion_bead}
In this paper, we initiated the instance-dependent analysis of the kernelized bandits problem. We first obtained a general complexity measure that characterizes the fundamental hardness of a specific problem instance. Then, we specialized this result to a smaller class of problems satisfying a local growth condition, to obtain explicit lower bounds in terms of the budget $n$. Finally, we introduced a new algorithm that achieves the best of both worlds: it matches the worst case performance limit~(modulo polylogarithmic terms) established by prior work, but also has the ability to adapt to the easier problem instances.  

The results of this paper lead to several interesting questions for future work, and we describe two key directions below:
\begin{itemize}
    \item A natural task is to investigate whether we can design an algorithm that is both minimax near-optimal and instance-optimal for the \matern family. More specifically, with $a_0 \defined (d+\nu)/(d+2\nu)$,  can we design an algorithm that satisfies $\sup_{f \in \rkhs[\kmat](M)}\mbb{E}[\mc{R}_n \lp \mc{A}, f \rp ] = \widetilde{\mc{O}}(n^{a_0})$, and $\mbb{E}\lb \mc{R}_n(\mc{A}, f)\rb = \widetilde{\mc{O}}\lp \comp(\Delta_n) \rp$ with $\Delta_n = n^{-(1-a)}$ for any $a>a_0$ simultaneously~(recall that $\comp$ denotes the complexity term introduced in~\Cref{def:complexity-term}). From a technical point of view, achieving this will be significantly aided by deriving tight time and $\X$-uniform confidence intervals for the GP model. 
    
    \item  Another interesting line of work is to adapt the ideas in our lower bound construction to problems such as kernel level-set estimation,  and optimization in other function spaces\citep{singh2021continuum, liu2021smooth}. The lower-bound technique of our paper can be easily generalized to these cases, as we discuss briefly in~\Cref{appendix:extensions}. However, designing algorithms that match the so-obtained instance-dependent lower bounds may require new techniques.
\end{itemize}

\newpage 
\bibliography{ref}

@article{kanagawa2018gaussian,
  title={Gaussian Processes and kernel methods: A review on connections and equivalences},
  author={Kanagawa, Motonobu and Hennig, Philipp and Sejdinovic, Dino and Sriperumbudur, Bharath K},
  journal={arXiv preprint arXiv:1807.02582},
  year={2018}
}

@article{srinivas2012information,
  title={{Information-theoretic regret bounds for Gaussian Process optimization in the bandit setting}},
  author={Srinivas, Niranjan and Krause, Andreas and Kakade, Sham M and Seeger, Matthias W},
  journal={IEEE Transactions on Information Theory},
  volume={58},
  number={5},
  pages={3250--3265},
  year={2012},
  publisher={IEEE}
}

@article{auer2002finite,
  title={Finite-time Analysis of the Multiarmed Bandit Problem},
  author={Auer, Peter and Cesa-Bianchi, Nicolo and Fischer, Paul},
  journal={Machine learning},
  volume={47},
  number={2-3},
  pages={235--256},
  year={2002},
  publisher={Springer}
}

@article{shekhar2018gaussian,
  title={{Gaussian process bandits with adaptive discretization}},
  author={Shekhar, Shubhanshu and Javidi, Tara},
  journal={Electronic Journal of Statistics},
  volume={12},
  number={2},
  pages={3829--3874},
  year={2018},
  publisher={The Institute of Mathematical Statistics and the Bernoulli Society}
}

@inproceedings{scarlett2017lower,
  title={Lower bounds on regret for noisy Gaussian Process bandit optimization},
  author={Scarlett, Jonathan and Bogunovic, Ilija and Cevher, Volkan},
  booktitle={Conference on Learning Theory},
  pages={1723--1742},
  year={2017}
}

@inproceedings{chowdhury2017kernelized,
  title={On kernelized multi-armed bandits},
  author={Chowdhury, Sayak Ray and Gopalan, Aditya},
  booktitle={Proceedings of the 34th International Conference on Machine Learning-Volume 70},
  pages={844--853},
  year={2017},
  organization={JMLR. org}
}

@inproceedings{janz2020bandit,
  title={Bandit optimisation of functions in the Mat{\'e}rn kernel RKHS},
  author={Janz, David and Burt, David and Gonz{\'a}lez, Javier},
  booktitle={International Conference on Artificial Intelligence and Statistics},
  pages={2486--2495},
  year={2020},
  organization={PMLR}
}

@inproceedings{wang2014bayesian,
  title={Bayesian multi-scale optimistic optimization},
  author={Wang, Ziyu and Shakibi, Babak and Jin, Lin and Freitas, Nando},
  booktitle={Artificial Intelligence and Statistics},
  pages={1005--1014},
  year={2014},
  organization={PMLR}
}

@inproceedings{munos2011optimistic,
  title={Optimistic optimization of a deterministic function without the knowledge of its smoothness},
  author={Munos, R{\'e}mi},
  booktitle={Advances in neural information processing systems},
  pages={783--791},
  year={2011}
}

@article{garivier2019explore,
  title={Explore first, exploit next: The true shape of regret in bandit problems},
  author={Garivier, Aur{\'e}lien and M{\'e}nard, Pierre and Stoltz, Gilles},
  journal={Mathematics of Operations Research},
  volume={44},
  number={2},
  pages={377--399},
  year={2019},
  publisher={INFORMS}
}

@inproceedings{valko2013finite,
  title={Finite-Time Analysis of Kernelised Contextual Bandits},
  author={Valko, Michal and Korda, Nathan and Munos, R{\'e}mi and Flaounas, Ilias and Cristianini, Nello},
  booktitle={Uncertainty in Artificial Intelligence},
  pages={654},
  year={2013}
}

@article{shekhar2020multi,
  title={Multi-scale zero-order optimization of smooth functions in an RKHS},
  author={Shekhar, Shubhanshu and Javidi, Tara},
  journal={arXiv preprint arXiv:2005.04832},
  year={2020}
}

@inproceedings{vakili2021information,
  title={On information gain and regret bounds in Gaussian process bandits},
  author={Vakili, Sattar and Khezeli, Kia and Picheny, Victor},
  booktitle={International Conference on Artificial Intelligence and Statistics},
  pages={82--90},
  year={2021},
  organization={PMLR}
}

@book{stein2012interpolation,
  title={Interpolation of spatial data: some theory for kriging},
  author={Stein, Michael L},
  year={2012},
  publisher={Springer Science \& Business Media}
}

@book{lattimore2020bandit,
  title={Bandit algorithms},
  author={Lattimore, Tor and Szepesv{\'a}ri, Csaba},
  year={2020},
  publisher={Cambridge University Press}
}

@article{bubeck2011x,
  title={X-armed bandits.},
  author={Bubeck, S{\'e}bastien and Munos, R{\'e}mi and Stoltz, Gilles and Szepesv{\'a}ri, Csaba},
  journal={Journal of Machine Learning Research},
  volume={12},
  number={5},
  year={2011}
}

@article{kleinberg2019bandits,
  title={Bandits and experts in metric spaces},
  author={Kleinberg, Robert and Slivkins, Aleksandrs and Upfal, Eli},
  journal={Journal of the ACM (JACM)},
  volume={66},
  number={4},
  pages={1--77},
  year={2019},
  publisher={ACM New York, NY, USA}
}

@techreport{polyanskiy2014lecture,
  title={Lecture notes on information theory},
  author={Polyanskiy, Yury and Wu, Yihong},
  year={2014},
  institution={MIT},
  note = "(\url{http://people.lids.mit.edu/yp/homepage/data/itlectures_v5.pdf})"
}

@techreport{van2014probability,
  title={Probability in high dimension},
  author={van Handel, Ramon},
  year={2014},
  institution={Princeton University NJ}, 
  note = "(\url{https://web.math.princeton.edu/~rvan/APC550.pdf})"
}

@article{bachoc2021instancedependent,
      title={Instance-Dependent Bounds for Zeroth-order Lipschitz Optimization with Error Certificates}, 
      author={François Bachoc and Tommaso Cesari and Sébastien Gerchinovitz},
      year={2021},
      journal={arXiv preprint arXiv:2102.01977},
      primaryClass={math.ST}
}

@article{wang2019optimization,
  title={Optimization of smooth functions with noisy observations: Local minimax rates},
  author={Wang, Yining and Balakrishnan, Sivaraman and Singh, Aarti},
  journal={IEEE Transactions on Information Theory},
  volume={65},
  number={11},
  pages={7350--7366},
  year={2019},
  publisher={IEEE}
}

@InProceedings{ramdas2013optimal,
  title = 	 {Optimal rates for stochastic convex optimization under Tsybakov noise condition},
  author = 	 {Ramdas, Aaditya and Singh, Aarti},
  booktitle = 	 {Proceedings of the 30th International Conference on Machine Learning},
  pages = 	 {365--373},
  year = 	 {2013},
  publisher =    {PMLR}
}

@article{castro2008minimax,
  title={Minimax bounds for active learning},
  author={Castro, Rui M and Nowak, Robert D},
  journal={IEEE Transactions on Information Theory},
  volume={54},
  number={5},
  pages={2339--2353},
  year={2008},
  publisher={IEEE}
}

@inproceedings{singh2021continuum,
  title={Continuum-armed bandits: a function space perspective},
  author={Singh, Shashank},
  booktitle={International Conference on Artificial Intelligence and Statistics},
  pages={2620--2628},
  year={2021},
  organization={PMLR}
}

@inproceedings{liu2021smooth,
  title={Smooth bandit optimization: generalization to Holder space},
  author={Liu, Yusha and Wang, Yining and Singh, Aarti},
  booktitle={International Conference on Artificial Intelligence and Statistics},
  pages={2206--2214},
  year={2021},
  organization={PMLR}
}


\newpage 
\begin{appendix}

\section{Proof of Theorem~\ref{theorem:general_lower-1}}
\label{proof:general_lower-1}

    \subsection{An intermediate one-step result}
    \label{subsec:intermediate}
        \begin{proposition}
        \label{theorem:general_lower}
            Consider the kernelized bandit problem with a budget $n$ and objective function $f \in \rkhs[\kmat](M)$ with $\|f\|_{\rkhs[\kmat]} = (1-\lambda)M$ for some $\lambda \in (0,1)$. 
            Let $\mc{A}$ denote an $a_0$-consistent algorithm for the class $\rkhs[\kmat](M)$, and fix an $a>a_0$. For constants $\Delta \geq 16n^{-(1-a)}$ and $c>1$, introduce the set $\mc{Z} = \{x \in \X: \Delta \leq f(x^*) - f(x) < c \Delta \}$, and with $w= \lp (c+1)\Delta M_\nu /(M \lambda) \rp^{1/\nu}$,  use $m(\mc{Z}, w)$ to denote the $2w$ packing number of $\mc{Z}$. 
           Then, the following is true for $n$ large enough~(exact condition in equation \ref{eq:n_large_enough1} below): 
            \begin{align}
                \label{eq:general_lower} 
                & \mbb{E} \lb \mc{R}_n \lp \mc{A}, f \rp \rb  \geq \frac{7 \log 2}{4} \;     \frac{ m(\mc{Z}, w)\,\sigma^2}{c^2 \Delta}.
            \end{align}
        \end{proposition}

        \begin{proof}
            Let $\{z_i: 1 \leq i \leq m(\mc{Z}, w)\}$ denote the points that form the maximal $2w$ packing set of $\mc{Z}$ with cardinality $m = m(\mc{Z}, w)$. By definition of $\mc{Z}$, the region $B(z_i, w)$ is at least $\Delta$-suboptimal for $f$. Building upon this fact, the proof of the result follows in these three steps: 
            \begin{itemize}
                \item First, we show that we can construct $m$ perturbed functions, denoted by $\{f_i: 1 \leq i \leq m(\mc{Z}, w)\}$, as introduced in~\Cref{def:perturbation}.  The function $f_i$ differs from $f$ only in the region $B(z_i, w)$, and in fact, it achieves its maximum value in that region. 
                
                \item Next, for each such perturbed function, we obtain a lower bound on the number of samples that the algorithm $\mc{A}$ must spend in $B(z_i, w)$.  
                
                \item Finally, the result follows by using the regret decomposition~\Cref{eq:regret_decomposition-1}, again using the fact that the points in $B(z, w)$ are $\Delta$-suboptimal for $f$. 
            \end{itemize}
        
            We now present the details of the steps outlined above. For every $i \in \{1, \ldots, m\}$, define the function  $f_i = f + g_{i}$, where $g_{i} = (c + 1)\Delta g\lp \frac{\cdot - z_i}{w} \rp$ is the bump function introduced in~\Cref{def:bump}. 
            Now the choice of the radius (or scale parameter) $\wni$ in the definition of $f_i$ implies that 
            \begin{align}
                \label{eq:thm1_proof_eq1}
                &\|g_{i} \|_{\rkhs[\kmat]} \leq \frac{ (c + 1)\Delta }{w^\nu} M_\nu \leq \lambda M. 
            \end{align}
        
            This implies the following: 
            \begin{itemize}
                \item The function $f_i$ satisfies $\|f_i\|_{\rkhs[\kmat]} \leq \|f\|_{\rkhs[\kmat]} + \|g_i\|_{\rkhs[\kmat]} \leq M$. Thus,  the function $f_i$ lies in the class $\rkhs[\kmat](M)$, and hence $\mc{A}$ achieves a regret $o\lp n^a \rp$ for any $a > a_0$ on $f_i$ for all $1 \leq i \leq m$. 
                
                \item The functions $f$ and $f_i$ have well separated optimal regions. More formally, if $x^*$ and $x^*_i$ denote the maximizers of $f$ and $f_i$ respectively, the following are true: 
                \begin{align}
                    &f(x^*) - f(x) \geq  \Delta, \quad \text{for all } x \in B(z_i, w), \\
                    &f_i(x^*_i) - f_i(x)  \geq \Delta,  \quad \text{for all } x \not \in B(z_i, w). 
                \end{align}
                
                \item The functions $f$ and $f_i$ differ from each other only in the region $B(z_i, w)$ and furthermore, they satisfy the following uniform deviation bound: 
                \begin{align}
                    \sup_{x \in \X} |f(x) - f_i(x)|\leq c\Delta.
                \end{align}
            \end{itemize}
 
            To summarize the above three points, the function $f_i$ is a $\big( B(z_i, w), \Delta, c \big)$-perturbation of $f$. Next, we show that any $a_0$-consistent algorithm must allocate at least a certain number of points to the region $B(z_i, w)$ when the true underlying function is $f$, in order to gather enough evidence to reject $f_i$. 
           
            \begin{lemma}
            \label{lemma:general_lower1}
                Let $N_i(\mc{A}, n)$ denote the number of times the algorithm $\mc{A}$ queries the oracle at points in the region $B(z_i, w)$. Then we have the following bound: 
                \begin{align}
                    & \mbb{E}_f \lb N_i(\mc{A}, n) \rb  \geq \frac{2 \sigma^2}{ c^2\Delta^2}\lp  \lp 1 - p_{n,i} \rp \log \lp \frac{1}{1 - q_{n,i}} \rp - \log 2 \rp, \quad \text{ where} \label{eq:general_lower1}  \\
                    & p_{n,i} \defined \frac{\mbb{E}_f \lb N_i(\mc{A}, n) \rb  }{n}  \quad \text{ and } \quad  q_{n,i} \defined \frac{\mbb{E}_{f_i} \lb N_i(\mc{A}, n) \rb  }{n}.
                \end{align}
                In the above display, $\mbb{E}_f$ denotes the expectation w.r.t. the probability measure induced by the pair $(f, \mc{A})$, and similarly $\mbb{E}_{f_i}$ denotes the probability measure induced by the pair $(f_i, \mc{A})$ for $1 \leq i \leq m_n$. 
            \end{lemma}
           
            The proof of~\eqref{eq:general_lower1} follows by relating the regret incurred by $\mc{A}$ on $f$ and $f_i$ respectively to a pair of multi-armed bandit problems with $(m_n+1)$ arms, and then applying the \emph{fundamental information inequality} \citep[\S~2]{garivier2019explore}.     The details of this proof are deferred to Appendix~\ref{appendix:general_lower2}
            
            Next, we simply the expression obtained in~\Cref{lemma:general_lower1} by appealing to the $a_0$-consistency of the algorithm $\mc{A}$. In the process, we also clarify  the meaning of the assumption that ``$n$ is large enough" in the statement of Theorem~\ref{theorem:general_lower}. In particular, we require that $n$ is large enough to ensure the following to hold simultaneously 
            \begin{align}
                &\mbb{E}_f \lb \mc{R}_n \lp \mc{A}, f \rp \rb \leq 2 n^a, \quad \text{ and }   \quad \mbb{E}_{f_i} \lb \mc{R}_n \lp \mc{A}, f_i \rp \rb \leq 2 n^a, \quad \text{ for all } i \in [m_n]. \label{eq:n_large_enough1}
            \end{align}
            
            Next, we observe that 
            \begin{align}
                1 - p_{n,i} & = 1 - \frac{ \mbb{E}_f \lb N_i(\mc{A}, n)\rb}{n}  = 1 - \frac{ \Delta \mbb{E}_f \lb N_i(\mc{A}, n)\rb }{\Delta n}  \\
                & \geq 1 - \frac{  \mbb{E}_f \lb \mc{R}_n(\mc{A}, f)\rb }{\Delta n}  \geq 1 - \frac{ \mbb{E}_f \lb \mc{R}_n(\mc{A}, f)\rb }{ 16 n^{1- (1-a)}}  \label{eq:step1} \\
                & \geq \frac 7 8. \label{eq:p_ni}
            \end{align}
            In the above display,\\
            \eqref{eq:step1} uses the fact that $  N_i \lp \mc{A}, n \rp  \Delta \leq \mc{R}_n \lp \mc{A}, f \rp$, and that $\Delta \geq 16 n^{-(1-a)}$,\\
            \eqref{eq:p_ni} uses the assumption made in~\eqref{eq:n_large_enough1} that $n$ is large enough to ensure that $\mbb{E}_0 \lb \mc{R}_n \lp \mc{A}, f \rp \rb \leq 2 n^a$. 
            
            Similarly, for the $q_{n,i}$ dependent term, we have 
            \begin{align}
                \frac{1}{1- q_{n,i}} & = \frac{n}{n - \mbb{E}_{f_i} N_i(\mc{A}, n)]} = \frac{n \Delta}{ \lp n - \mbb{E}_{f_i} N_i(\mc{A}, n)] \rp  \Delta}  
                 \geq \frac{ n \Delta}{ \mbb{E}_{f_i} \lb \mc{R}_n \lp \mc{A}, f_i \rp \rb }  \geq \frac{ n \Delta} { 2 n^a } \label{eq:step3}  \\
                &= \frac{ n^{1-a} \Delta}{2}.  \label{eq:q_ni}
            \end{align}
            In the above display,~\eqref{eq:step3} uses the assumption  that $n$ is large enough to ensure that $\mbb{E}_{f_i} \lb \mc{R}_n \lp \mc{A}, f_i \rp \rb \leq 2 n^a$. 
           
            Putting~\eqref{eq:p_ni} and~\eqref{eq:q_ni} back in~\eqref{eq:general_lower1}, we get 
            \begin{align}
                \mbb{E}_f \lb N_i \lp \mc{A}, n \rp \rb \geq \frac{2 \sigma^2}{  c^2 \Delta^2} \lp \frac 7 8 \log \lp \frac{ \Dni n^{1-a}}{2\, 2^{8/7}} \rp \rp   \geq \frac{7 \log 2 }{4} \frac{ \sigma^2}{c^2 \Delta^2}.  
            \end{align}
            
            Finally, the result stated in~\eqref{eq:general_lower} follows by repeating the argument of Lemma~\ref{lemma:general_lower1} for all the different values of $i \in \{1, \ldots, m\}$, and noting that $\mbb{E}_f \lb \mc{R}_n \lp \mc{A}, f \rp \rb \geq \sum_{i=1}^{m} \Delta \mbb{E}_{f} \lb N_i(\mc{A}, n ) \rb$ from the decomposition inequality~\eqref{eq:regret_decomposition-1}. 
        \end{proof}

    \subsubsection{Proof of Lemma~\ref{lemma:general_lower1}}
    \label{appendix:general_lower2}
        To prove this result, we need to introduce some additional notation. We use $\mc{H}_t$ to denote the observations up to, and including, time $t$ for $t \in \{1, 2, \ldots, n\}$. For a given $n \geq 1$, introduce the sample space $\Omega = \lp \X \times \Y\rp^{n}$, and let $\mc{F}_0$ denote a sigma algebra of subsets of $\Omega$. For a given function $f$ and an querying strategy $\mc{A}$, we use $\mbb{P}^{(f)}_{\mc{A}}$ to denote the probability measure on $\Omega$ induced by the pair $(f, \mc{A})$. We will drop the $\mc{A}$ dependence of $\mbb{P}^{(f)}_{\mc{A}}$, and simply use $\mbb{P}^{(f)}$ in the sequel. 
        
        The first step is to obtain an upper bound on the KL-divergence between the measures $\mbb{P}^{(f)}$ and $\mbb{P}^{(f_i)}$ induced on the space $\Omega$, for a common algorithm $\mc{A}$. In particular, suppose $(X_1, Y_1, \ldots, X_n, Y_n)$ denote the query-observation pairs collected by the algorithm $\mc{A}$ up to time $n$. Then we have the following: 
        \begin{align}
            \mc{D}_n &\defined  D_{KL} \lp \mbb{P}^{(f)}, \, \mbb{P}^{(f_i)} \rp   \\
            & = \mc{D}_{n-1} + D_{KL} \lp \mbb{P}_{X_n, Y_n \vert \mc{H}_{n-1}}^{(f)}, \mbb{P}_{X_n, Y_n \vert \mc{H}_{n-1}}^{(f_i)} \vert \mc{H}_{n-1} \rp \label{eq:1_lemma_proof}\\ 
            & = \mc{D}_{n-1} + D_{KL} \lp \mbb{P}^{(f)}_{X_n \vert \mc{H}_{n-1}}, \mbb{P}^{(f_i)}_{X_n \vert \mc{H}_{n-1}} \rp + 
            D_{KL} \lp \mbb{P}_{Y_n \vert \mc{H}_{n-1}, X_n}^{(f)}, \mbb{P}_{Y_n \vert \mc{H}_{n-1}, X_n}^{(f_i)} \vert \mc{H}_{n-1}, X_n \rp \label{eq:2_lemma_proof}\\ 
            & = \mc{D}_{n-1} + 0 +      D_{KL} \lp \mbb{P}_{Y_n \vert \mc{H}_{n-1}, X_n}^{(f)}, \mbb{P}_{Y_n \vert \mc{H}_{n-1}, X_n}^{(f_i)} \vert \mc{H}_{n-1}, X_n \rp \label{eq:3_lemma_proof}\\ 
            & = \mc{D}_{n-1} + \mbb{E}_f \lb \frac{ \lp f(X_n) - f_i(X_n) \rp^2}{2 \sigma^2} \rb \label{eq:4_lemma_proof}\\
            & \leq \mc{D}_{n-1} + \mbb{E}_f \lb \indi{X_n \in B(z_i, w)} \frac{ c^2 \Delta^2}{2 \sigma^2} \rb \label{eq:5_lemma_proof}
        \end{align}
        
        In the above display, 
        \begin{itemize}
            \item 
        \eqref{eq:1_lemma_proof} and~\eqref{eq:2_lemma_proof} follow from the chain rule for KL-divergence~\citep[Theorem~2.2]{polyanskiy2014lecture}, 
        \item 
        \eqref{eq:3_lemma_proof} uses the fact that conditioned on $\mc{H}_{n-1}$, the distribution of $X_n$ is the same for both the problem instances, that is they are both selected according to the mapping $A_n : \lp \mc{X} \times \mc{Y} \rp^{n-1} \mapsto \mc{X}$, where $\mc{A} = \lp A_t \rp_{t=1}^n$ is the common strategy,
        \item \eqref{eq:4_lemma_proof} uses the fact that condition on $X_n$, $Y_n$ is distributed as $N\lp f(X_n), \sigma^2\rp$ and $N\lp f_i(X_n), \sigma^2 \rp$ under the two distributions $\mbb{P}^{(f)}$ and $\mbb{P}^{(f_i)}$ respectively, and 
        \item \eqref{eq:5_lemma_proof} uses the fact that, by construction, $f$ and $f_i$ only differ in the region $B(z_i, w)$, and furthermore, in this region we have $\max_{x \in B(z_i, w)} |f(x) - f_i(x)|\leq c \Delta$
        \end{itemize}
        
        Repeating the steps involved in obtaining~\eqref{eq:5_lemma_proof} $n-1$ times, we get the following upper bound on the term $\mc{D}_n$: 
        \begin{align}
            \mc{D}_n & \leq \frac{ c^2 \Delta^2}{2 \sigma^2} \sum_{t=1}^n \mbb{E}_f \lb \indi{X_t \in B\lp z_i, w \rp } \rb = \frac{ c^2 \Delta^2}{2 \sigma^2}\mbb{E}_f \lb \sum_{t=1}^n \indi{X_t \in B\lp z_i, w \rp } \rb  \\
            & = \frac{ c^2 \Delta^2}{2 \sigma^2} \; \mbb{E}_f \lb N_i\lp \mc{A}, n \rp \rb. 
        \end{align}
        Recall that the term $N_i(\mc{A}, n)$ denotes the number of times the algorithm $\mc{A}$ queries points from the region $B(z_i, w)$ in the first $n$ rounds. 
        
        Now, suppose $Z: \Omega \mapsto [0,1]$ be any measurable $[0,1]$ valued random variable. Then by \citep[Lemma~1]{garivier2019explore}, we get the following result: 
        \begin{align}
            D_{KL} \lp \mbb{P}^{(f)}, \, \mbb{P}^{(f_i)} \rp & \geq kl \lp \mbb{E}_f \lb Z \rb, \, \mbb{E}_{f'} \lb Z \rb \rp, 
        \end{align}
        where $kl(p, q)$ for $p,q \in [0,1]$ denotes the KL-divergence between two Bernoulli random variables with means $p$ and $q$ respectively. 
        
        To complete the proof, we select $Z \defined \frac{ N_i\lp \mc{A}, n \rp   }{n}$, and using the fact~\citep[Eq.~(11)]{garivier2019explore} that $kl(p, q) \geq -(1-p) \log (1-q) - \log 2$, we get the required inequality 
        \begin{align}
            & \mbb{E}_f \lb N_i (\tilde{\mc{A}}, n ) \rb \frac{ c^2 \Dni^2}{2 \sigma^2}  \geq - \lp 1- p_{n,i} \rp \log \lp 1 - q_{n,i} \rp  - \log 2,\quad \text{where} \\
          & p_{n,i} = \mbb{E}_f \lb Z \rb, \quad \text{ and } q_{n,i} = \mbb{E}_{f_i} \lb Z \rb.  
        \end{align}
        
        
        \begin{remark}
        \label{remark:gaussian_noise}
            The only point at which we exploit the assumption that the observation noise is distributed as $N(0, \sigma^2)$ (i.e., Assumption~\ref{assump:noise-lower}) is in obtaining the inequality~\eqref{eq:4_lemma_proof}. Due to this assumption on the noise, we get a closed form expression for an upper bound on the KL-divergence in~\eqref{eq:5_lemma_proof}, i.e., $\lp c^2 \Dni^2 \rp / 2 \sigma^2$. In general, if we only assumed that the observation noise was $\sigma^2$ sub-Gaussian, then the same result would hold true with the previous closed-form upper bound replaced by the expression $\sup_{x \in B(z_i, w)}\; D_{KL} \lp \mbb{P}^{(f)}_{Y_n \vert X_n = x}, \, \mbb{P}^{(f_i)}_{Y_n \vert X_n = x} \rp$. 
        \end{remark}
    \subsection{Concluding~Theorem~\ref{theorem:general_lower-1} from~Proposition~\ref{theorem:general_lower}}
    \label{subsec:obtianing-general}
        \Cref{theorem:general_lower-1} follows by repeated application of the result in~\Cref{theorem:general_lower} to different regions of the input space. More specifically, introduce the following notation: 
        \begin{itemize}
            \item As in the previous section, we fix an $a>a_0$, and choose $c=2$ and  $\Delta = 16n^{-(1-a)}$. 
            \item For $k \geq 0$, set $\mc{Z}_k = \{x \in \X: 2^k\Delta \leq f(x^*) - f(x) < 2^{k+1}\Delta \}$. With $w_k \defined \lp (3\times2^k\Delta M_\nu)/(\lambda M) \rp^{1/\nu}$, we  use $m_k$ to denote the $2w_k$ packing number of $\mc{Z}_k$. 
        \end{itemize}
        
        Since, $f \in \rkhs[\kmat](M)$, we know that the set $\mc{Z}_k$ is an empty set for all $k$ larger than a finite value $k_0$. More specifically, we have $k_0 = \left \lceil \log_2 \lp  \frac{ \sup_{x\in \X}f(x^*) - f(x)}{\Delta} \rp \right \rceil \leq \left \lceil \log_2 \lp \frac{ 2M\kmat(0)}{\Delta} \rp \right \rceil$. 
       
       Finally, the statement of~\Cref{theorem:general_lower-1} follows by $k_0+1$ repeated applications of the intermediate statement proved in~\Cref{theorem:general_lower} using $\mc{Z} = \mc{Z}_k$, $\Delta = 2^k\Delta$, $c=2$, and $w_k = \lp (3\times2^k\Delta M_\nu)/(\lambda M) \rp^{1/\nu}$ for $k=0, 1, \ldots, k_0$.

\section{Proof of Theorem~\ref{theorem:lower_local}}
\label{appendix:lower_local}
    To prove this statement, we appeal to the one-step result obtained in~\Cref{theorem:general_lower}. In particular, we apply~\Cref{theorem:general_lower} with the following parameters: 
    \begin{itemize}
        \item We set $\Delta = \Delta_n = 16n^{-(1-a)}$, and $c = c_n \defined \frac{2^{1/d} \cu}{\cl}$ where $\cu$ and $\cl$ are the parameters introduced in~\Cref{assump:growth}. 
        \item The set $\mc{Z}$ of~\Cref{theorem:general_lower} now becomes $\{x \in \mc{X}: \Delta_n \leq f(x^*)-f(x) <c_n \Delta_n\}$.  
        \item We set the radius of the balls to $w = w_n = \frac{ (c_n+1) \Delta_n M_{\nu}}{M \lambda}$, where $\lambda = 1 - \|f\|_{\rkhs[\kmat]}/M$, and use $m(\mc{Z}, w_n)$ to denote the $2w_n$ packing number of the set $\mc{Z}$. 
    \end{itemize}
    
    With these parameters,~\Cref{theorem:general_lower} gives us the following lower bound on the regret: 
    \begin{align}
        \label{eq:proof-local-1}
        \mbb{E}\lb \mc{R}_n \lp \mc{A}, f \rp \rb = \Omega \lp \frac{ \sigma^2 m(\mc{Z}, w_n)}{\Delta_n c_n^2} \rp. 
    \end{align}
    
    To conclude the statement of~\Cref{theorem:lower_local}, we will show that $m = m(\mc{Z}, w_n) = \Omega \lp  \Delta_n^{\frac{d}{\nu} \lp 1 - \frac{\nu}{b} \rp} \rp$. First, we introduce the following terms:
    \begin{align}
        \label{eq:1_thm1}
        r_0 \defined \lp \frac{\Dn}{\cl} \rp^{1/b}, \qquad r_1 \defined \lp \frac{c_n \Dn}{\cu} \rp^{1/b} \quad \text{where }\;\; c_n = \frac{2^{1/d} \cu}{\cl}\quad \text{as before}. 
    \end{align}
    
    Next, we use~\Cref{assump:growth} to obtain the following result about $\mc{Z}$. 
 
    \begin{lemma}
    \label{lemma:step1_thm1}
        With $r_0$ and $r_1$ introduced in~\eqref{eq:1_thm1} and $\widetilde{\X}(f, \Delta, c)$ defined above,  we have 
        \begin{align}
            \label{eq:2_thm1}
            \mc{Z} \; \supset \; B(x^*, r_0, r_1) \; \defined \; \{x \in \X : r_0 \leq \|x-x^*\| < r_1 \} 
        \end{align}
    \end{lemma}
    
    \begin{proof}
        Suppose $x \in B(x^*, r_0, r_1)$. Then we have the following: 
        \begin{align}
            &\|x-x^*\| \geq r_0 \defined \lp \frac{\Dn}{\cl} \rp^{1/b}\\
            \Rightarrow\; & f(x^*) -f(x) \geq \cl r_0^b \\
            \Rightarrow\; & f(x^*) -f(x) \geq \Dn. \label{eq:2_thm10}
        \end{align}
    Similarly, we also have the following: 
    \begin{align}
        & \|x-x^*\| \leq r_1 \defined \lp \frac{ c_n \Dn}{\cu} \rp^{1/b}, \\
        \Rightarrow\; & f(x^*) - f(x) \leq \cu r_1^b \\
        \Rightarrow\; & f(x^*) - f(x) \leq \cu \lp \frac{ c_n \Dn}{\cu} \rp  \\
        \Rightarrow\; & f(x^*) - f(x) \leq  c_n \Dn. \label{eq:3_thm1}
    \end{align}
    Together,~\eqref{eq:2_thm10} and~\eqref{eq:3_thm1} imply that if $x \in B(x^*, r_0, r_1)$ then $x \in \mc{Z}$. 
    \end{proof}
    
    The above statement implies that the $2w_n$ packing number of $\mc{Z}$ can be lower-bounded by the $2w_n$ packing number of the smaller set $B(x^*, r_0, r_1)$. Let us denote the $2w_n$-packing number of $B(x^*, r_0, r_1)$ with $\widetilde{m}$. Then, by using the fact that the $2w_n$ packing number is lower bounded by the $2w$ covering number, and employing the standard volume arguments~\citep[Lemma~5.13]{van2014probability}, we conclude that there exists a constant $0<C_1 < \infty$ such that 
    \begin{align}
        m(\mc{Z}, w_n) \geq m\lp B(x^*, r_0, r_1), w_n\rp \geq C_1 \lp \frac{ \Dn^{1/b}}{\Dn^{1/\nu}} \rp^{d} = C_1 \Dn ^{\frac{d}{\nu}\lp 1 - \frac{\nu}{b} \rp }. 
    \end{align}
    Plugging this back in~\eqref{eq:proof-local-1} gives us the required result.

\section{Proof of Theorem~\ref{theorem:bead_upper}}
\label{appendix:bead_upper}
First, we introduce  a general class of kernels~(that includes the \matern family), for which our regret bound will be valid. 
\begin{definition}
\label{def:kernel_class} 

We use $\mc{K}$ to represent the class of isotropic kernel functions ~(i.e, $K$ such that $K(x,z)$ depends only on $\|x-z\|$) which satisfy the property that $\sqrt{K(x,z) + K(z,x) - 2(x,z)} \leq C_K \|x-z\|^\xi$ for some $C_K>0$ and $\xi \in(0, 1]$ for all $x, z \in \X$.
\end{definition}

Next, we present a simple embedding result which is crucial in the adaptive partitioning approach used in our algorithms. 
\begin{proposition}
\label{prop:embed}
If a function $f \in \rkhs(M)$ for some $0<M<\infty$ and $K \in \mc{K}$, then we have $\|f(x) - f(z) \| \leq M C_K \|x - z\|^\xi$ for all $x, z \in \X$.  In particular, this H\"older smoothness property is satisfied by elements of \matern RKHS~($\kmat$) with $\xi = \min \{ \nu, 1\}$. 
\end{proposition}

We begin with the following independence result about the points queried by the algorithm. 

\begin{proposition}
\label{prop:conditional_indep} 
Suppose $\mc{P}_t \subset \mc{X}_h$ is the active set of points at some time $t$. Let $t_0<t$ denote the time at which a point from $\mc{P}_t$ was first queried, and let $\mc{E}_t$ denote the multi-set of points $\{x_{t_0}, \ldots, x_{t-1}\}$ queried by the algorithm. Then the collection of random variables $(y_{t_0}, \ldots, y_{t-1} )$ are mutually independent, conditioned on the observations $x_{t_0}, \ldots, x_{t-1}$. 
\end{proposition}

\begin{proof}
The proof proceeds as follows: 
\begin{align}
    P\lp y_{t_0}, \ldots, y_{t-1} \vert x_{t_0}, \ldots, x_{t-1} \rp &= \frac{P \lp y_{t_0}, \ldots, y_{t-1}, x_{t_0}, \ldots, x_{t-1} \rp }{P\lp y_{t_0}, \ldots, y_{t-1} \rp} \\
    & =\frac{ \prod_{s=t_0}^{t-1} P\lp x_{s} \vert x_{t_0}, \ldots, x_{s-1}, y_{t_0}, \ldots, y_{s-1} \rp  P \lp y_{s} \vert x_{t_0}, \ldots, x_s, y_{t_0}, \ldots, y_{s-1}  \rp }{ P\lp x_{t_0}, \ldots, x_{t-1} \rp } \\
    & \stackrel{(a)}{=} \frac{ \prod_{s=t_0}^{t-1} P\lp x_{s} \vert x_{t_0}, \ldots, x_{s-1}\rp  P \lp y_{s} \vert x_s \rp }{ P\lp x_{t_0}, \ldots, x_{t-1} \rp } \\
    & \stackrel{}{=} \frac{   P\lp x_{t_0}, \ldots x_{t- 1} \rp \lp \prod_{s=t_0}^{t-1} P \lp y_{s} \vert x_s \rp \rp }{ P\lp x_{t_0}, \ldots, x_{t-1} \rp } \\
    & \stackrel{(b)}{=} \prod_{s=t_0}^{t-1} P\lp y_s | x_{t_0}, \ldots, x_{t-1} \rp. 
\end{align}
In the above display,\\
\tbf{(a)} uses the fact that at any $s \geq t_0$, the query point $x_s$ only depends on the previous query points $x_{t_0}, \ldots, x_{s-1}$ and not on the observations; and the fact that conditioned on $x_s$, the observation $y_s$ is independent of the query points $x_{t_0}, \ldots, x_{s-1}$ and observations $y_{t_0}, \ldots, y_{s-1}$. \\
\tbf{(b)}  uses the fact that conditioned on $x_s$ the observation $y_s$ is independent of $x_{t_0}, \ldots, x_{s-1}, x_{s+1}, \ldots, x_{t-1}$. 

The equality \tbf{(b)} implies the conditional independence of the observations given the query points belonging to the current active set $\mc{P}_t$, as required. 
\end{proof}

Having obtained the conditional independence property of the query points, we now present the key concentration result that leads to the required regret bounds. 

\begin{lemma}
\label{lemma:concentration} 
For some $t \geq 1$, let $\mc{P}_t$, $t_0$ and $\mc{E}_t$ be the same as in Proposition~\ref{prop:conditional_indep}. Then, for a given $\delta \in (0,1)$, the following is true: 
\begin{align}
    & P \bigg( \exists x \in \mc{P}_t, \text{ s.t. } |f(x) - \mu_t(x)| > \beta_t \sigma_t(x) \bigg) \leq \delta_t \defined \frac{6 \delta}{t^2 \pi^2 }, \label{eq:concentation1} \\
    \text{ where } & \beta_t = \sqrt{ 2 \sigma^2 \log \lp \frac{ |\mc{P}_t| \pi^2 t^2}{3 \delta}\rp } \label{eq:beta_t}  
\end{align}
Recall that the terms $\mu_t(\cdot)$ and $\sigma_t(\cdot)$ represent the posterior mean and standard-deviation functions computed by the subroutine \posterior introduced in Definition~\ref{def:compute_posterior}. 
\end{lemma}

\begin{proof}
To prove this result, we rely on the following facts derived by \citet{valko2013finite} while proving their Lemma~2. For $x \in \mc{P}_t$, there exists $A_{t_0}, \ldots, A_{t-1} \in \mbb{R}$~(depending on $x$) such that the following holds: 

\begin{align}
& f(x) - \mu_t(x) = \sum_{i=t_0}^{t-1} A_i \lp y_i - f(x_i) \rp + B_x \\
\text{ with }  & \sum_{i=t_0}^{t-1} A_i^2 \leq \sigma_t(x)^2   \quad \text{ and } \quad |B_x| \leq \tau^{-1/2} M \sigma_t(x),  
\end{align}
for all $x \in \mc{P}_t$. Recall that $\tau$ is the regularization parameter used in the subroutine \posterior, while $M$ is the upper bound on the RKHS norm of $f$. 

Now, we use the conditional independence property derived in Proposition~\ref{prop:conditional_indep} along with the conditional $\sigma^2-$sub-Gaussianity of the observation noise to get the required concentration result. 

In particular, for a given $t \geq 1$ and a fixed $x \in \mc{P}_t$, we have 
\begin{align}
   P \lp  f(x) - \mu_t(x) > \beta_t \sigma_t(x) \rp & = \mbb{E}\lb P\lp \sum_{i=t_0}^{t-1} A_i \lp y_i - f(x_i) \rp > \beta_t \sigma_t(x) | \lp x_i \rp_{i=t_0}^{t-1} \rp \rb \label{eq:conc1}\\
   & \leq \mbb{E} \lb \mbb{E} \lb \exp \lp \lambda \sum_{i=t_0}^{t-1}A_i (y_i-f(x_i))  \rp e^{-\lambda \beta_t \sigma_t(x)} \left. \right\vert \lp x_i \rp_{i=t_0}^{t-1} \rb \rb \label{eq:conc2}\\
   & = \mbb{E} \lb \prod_{i=t_0}^{t-1} \mbb{E} \lb \exp \lp \lambda A_i \lp y_i - f(x_i) \rp \rp \vert (x_i)_{i=t_0}^{t-1} \rb  \rb \label{eq:conc3}\\
   & \leq \exp \lp \frac{\lambda^2 \sigma^2}{2} \sum_{i=t_0}^{t-1}A_i^2 - \lambda \beta_t \sigma_t(x) \rp \label{eq:conc4}\\
   & \leq \exp \lp \frac{\lambda^2 \sigma^2}{2} \sigma_t(x)^2 - \lambda \beta_t \sigma_t(x) \rp \label{eq:conc5}\\
   & = \exp \lp - \frac{ \beta_t^2}{2 \sigma^2} \rp. \label{eq:conc6}
\end{align}
In the above display, \\
\eqref{eq:conc2} follows by an application of Chernoff's inequality with some constant $\lambda>0$ to be selected later, \\
\eqref{eq:conc3} follows from the conditional independence property derived in Proposition~\ref{prop:conditional_indep}, \\
\eqref{eq:conc4} uses the fact that, conditioned on $x_i$, the random variable $y_i - f(x_i)$ is zero-mean $\sigma^2$ sub-Gaussian, \\
\eqref{eq:conc5} uses the fact that $\sum_{i=t_0}^{t-1}A_i^2 \leq \sigma_t(x)^2$, and \\
\eqref{eq:conc6} follows by selecting $\lambda = \beta_t/(\sigma^2 \sigma_t(x))$. 

Repeating the argument of the previous display with $\mu_t(x) - f(x)$, in the place of $f(x) - \mu_t(x)$, gives us that 
\begin{align}
    P \lp |f(x) - \mu_t(x)| > \beta_t \sigma_t(x) \rp \leq 2 \exp \lp - \frac{\beta_t^2}{2\sigma^2} \rp. 
\end{align}

Next, using the fact that $\beta_t > \sqrt{ 2 \sigma^2 \log \lp \frac{|\mc{P}_t|\pi^2 t^2 }{3 \delta}\rp }$, we have by a union bound over $x \in \mc{P}_t$: 
\begin{align}
    P \lp \exists x \in \mc{P}_t, \text{ s.t. } |f(x) - \mu_t(x)| > \beta_t \sigma_t(x) \rp  & \leq \sum_{x \in \mc{P}_t} 2  \exp \lp - \frac{\beta_t^2}{2\sigma^2} \rp \\
    & < 2 \exp \lp - \log \lp \frac{ |\mc{P}_t| \pi^2 t^2}{3 \delta} \rp \rp \\
    & =   \frac{6 \delta}{\pi^2 t^2} \defined \delta_t. 
\end{align}
\end{proof}

Since $\sum_{t=1}^n \delta_t < \delta$, we note that the following event $\mc{E}$, occurs with probability at least $1-\delta$ 
\begin{align}
\label{eq:event_conc}
    \mc{E} \defined \cap_{t =1 }^n \cap_{x \in \mc{P}_t} \{ |f(x) - \mu_t(x)| \leq \beta_t \sigma_t(x) \}. 
\end{align}
Throughout the rest of the proof, we will work under the event $\mc{E}$ with $\delta = 1/n$. Hence, the expected regret of the algorithm $\mc{A}_1$ can then be upper bounded by 
\begin{align}
\label{eq:regret-event}
    \mbb{E}[\mc{R}_n (\mc{A}_1, f)] &= \mbb{E}[\mc{R}_n (\mc{A}_1, f) \indi{\mc{E}}] + \mbb{E}[\mc{R}_n (\mc{A}_1, f) \indi{\mc{E}^c}] \\
    & \leq \mbb{E}[\mc{R}_n (\mc{A}_1, f) \indi{\mc{E}}]  + \mbb{P}\lp \mc{E}^c \rp \times n \times \sup_{x \in \X} (f(x^*) -f(x)) \\ 
    & \leq \mbb{E}[\mc{R}_n (\mc{A}_1, f) \indi{\mc{E}}]  + 2 \|f\|_{\rkhs[\kmat]} \text{diam}(\X) \\
    & \leq \mbb{E}[\mc{R}_n (\mc{A}_1, f) \indi{\mc{E}}]  + \mc{O}(1). 
\end{align}
Since the second term is upper bounded by a constant, it suffices to show that the required upper bounds hold for the regret incurred under the event $\mc{E}$. 

We now obtain a result about the sub-optimality of the points queried by the algorithm. 
\begin{lemma}
\label{lemma:suboptimality}
Suppose event $\mc{E}$ introduced in~\eqref{eq:event_conc} occurs, and the algorithm queries a point $x_t \in \mc{P}_t \subset \mc{X}_h$ for some $h \geq 1$. Then we have
\begin{align}
    f(x^*) - f(x_t) \leq \lp 7 L v_1^\xi \rho^{-\xi} \rp \rho^{h \xi} = \mc{O} \lp \rho^{h \xi} \rp. 
\end{align}
\begin{proof}
Since $h \geq 1$, the set $\mc{P}_t$ must have been formed by a call to the \refine subroutine. Let $x_t = x_{h,i}$ for some $i \in \{1,\ldots, 2^h\}$ and furthermore, denote its parent node by $x_{h',i'}$ where $h'=h-1$ and $i'=\lceil i/2 \rceil$. Assume that the active set $\mc{P}_t$ was formed by a call to the \refine subroutine at some time $t_0 < t$. Then the following must be true: 
\begin{align}
    f(x_{h,i}) = f(x_t) & \geq f(x_{h',i'}) - 2 \overbrace{L (v_1\rho^{h'})^\xi}_{\defined V_{h'}} \geq \mu_{t_0}(x_{h',i'}) - 2\beta_{t_0} \sigma_{t_0} (x_{h',i'}) - V_{h'} \label{eq:subopt1}\\ 
    & \geq \mu_{t_0}(x_{h',i'}) + \beta_{t_0} \sigma_{t_0} (x_{h',i'}) - 3 V_{h'} \label{eq:subopt2} \\
    & \geq \max_{x  \in \mc{P}_{t_0}} \lp \mu_{t_0}(x) - \beta_{t_0} \sigma_{t_0}(x) \rp - 3V_{h'} \label{eq:subopt3}\\
    & \geq \lp f(x^*) - 4V_{h'} \rp - 3V_{h'} 
     = f(x^*) - 7V_{h'}. \label{eq:subopt4} 
\end{align}
In the above display, 
\begin{itemize}
    \item 
the first inequality in \eqref{eq:subopt1} uses the fact that $f$ is $(L, \xi)$ H\"older continuous, while that second inequality uses the fact that under the event $\mc{E}$, we have $f(x_{h', i'}) \geq \mu_{t_0}(x_{h',i'}) - \beta_{t_0} \sigma_{t_0}(x_{h',i'})$, 
\item 
\eqref{eq:subopt2} follows by adding and subtracting $2\beta_{t_0}\sigma_{t_0}(x_{h',i'})$, and then using the fact that $2 \beta_{t_0} \sigma_{t_0}(x_{h',i'})$ must be smaller than $2V_{h'} \defined L (v_1\rho^{h'})^\xi$, due to line~6 in Algorithm~\ref{algo:bead},
\item  \eqref{eq:subopt3} then uses the fact $\mu_{t_0}(x_{h',i'})  + \beta_{t_0} \sigma_{t_0}(x_{h',i'})$ must be larger than the highest lower bound, $\max_{x \in \mc{P}_{t_0}} \mu_{t_0}(x) - \beta_{t_0} \sigma_{t_0}(x)$ in order for $x_{h,i}$ to be included in the updated $\mc{P}_{t_0}$ returned by \refine, 
\item and finally, \eqref{eq:subopt4} uses the fact that $f(x^*) - 2V_{h'} \geq \max_{x \in \mc{P}_{t_0}} \mu_{t_0}(x) - \beta_{t_0} \sigma_{t_0}(x) $. To see this,  suppose $x_{h', j}$ denotes the point in $\mc{P}_{t_0}$ such that $x^* \in \X_{h',j}$. Then we must have the following:
\begin{align}
   \max_{x \in \mc{P}_{t_0}} \mu_{t_0}(x) - \beta_{t_0} \sigma_{t_0}(x) \geq & \mu_{t_0}\lp x_{h', j} \rp - \beta_{t_0} \sigma_{t_0} \lp x_{h', j} \rp \geq f(x_{h', j}) - 2 \beta_{t_0}(x_{h', j})  \\
    \geq & f(x^*) - 2V_{h'} - 2 \beta_{t_0}(x_{h', j})  
    \geq  f(x^*) - 4V_{h'}. 
\end{align}
\end{itemize}
\end{proof}
\end{lemma}

The previous lemma, gives us a bound on the suboptimality of any point queried by the algorithm at time $t$ in terms of the parameter $\rho$ and the depth of the cell $h$. Now, let $N_h$ denote the number of times the algorithm queries a point at level $h$, i.e., lying in the subset $\mc{X}_h$. Then we have the following regret decomposition (assuming event $\mc{E}$ defined in ~\eqref{eq:event_conc} occurs): 
\begin{align}
   \mc{R}_n(\mc{A}_1, f) = \sum_{t=1}^n f(x^*) - f(x_t) = \mc{O}\lp  \sum_{h \geq 0} N_h \rho^h \rp . 
\end{align}

To complete the proof, it remains to get an upper bound on the term $N_h$, which we do in two ways: one in terms of the maximum information gain $\igain$, and the other in terms of the upper-complexity term~$\compupper$ introduced in~\Cref{def:upper-complexity}. 


We now present the $\igain$ based upper bound on $N_h$. As an immediate consequence of this bound, we also observe that $\mc{A}_1$ is minimax near-optimal. 
\begin{lemma}
\label{lemma:regret_bound1}
The number of queries made by $\mc{A}_1$ at level $h$ of the tree satisfies $N_h = \widetilde{\mc{O}} \lp \rho^{-2h\xi} \igain \rp$. As a consequence of this, we obtain the following upper bound on the regret: 
\begin{align}
    \label{eq:regret_bound1}
    \mbb{E} \lb \mc{R}_n(\mc{A}_1, f) \rb = \tOh{ \sqrt{n \igain} } = \tOh{ n^{a_{\nu}^*}}, \quad \text{with} \quad a_{\nu}^* = \frac{\nu+d}{2\nu+d}. 
\end{align}
\end{lemma}

\begin{proof}
This result follows by using \citep[Lemma~4~and~5]{valko2013finite} to get that $N_h = \tOh{ \rho^{-h\xi} \beta_n \sqrt{ \igain N_h} }$. Dividing both sides by $\sqrt{N_h}$ and taking the square gives the upper bound on $N_h$. 

Next, under event $\mc{E}$, we have $\mc{R}_n = \mc{O} \lp \sum_{h = 0}^{h_{\max}} \rho^{h \xi} N_h \rp = \tOh{ \sum_{h=0}^{h_{\max}} \beta_n \sqrt{\igain N_h}  } = \tOh{ h_{\max} \sqrt{\igain n}  } = \tOh{ \sqrt{\igain n}}$.  In the last two equalities, we used the fact that $\beta_n = \mc{O}\lp \sqrt{\log n}\rp$ and $h_{\max} = \log n$, and hence are absorbed by the hidden polylogarithmic leading constant in the notation $\tOh{\cdot}$. 

\end{proof}

\begin{lemma}
\label{lemma:regret_bound2}
Assume that the event $\mc{E}$ holds, and let $N_h$ denote the number of queries made by $\mc{A}_1$ at level $h$.  Introduce the set $W_h \defined \{x \in \X: f(x^*) - f(x) \leq (7Lv_1^{\xi} \rho^{-\xi}) \rho^{h\xi} \}$, and let $m_h = m\lp W_h, 2v_2 \rho^h \rp $ denote the $2 v_2\rho^h$-packing number of the set $W_h$. Then, $N_h$ is upper bounded by $\tOh{ \rho^{-2h\xi} m_h}$. 
\end{lemma}

\begin{proof}
Next, suppose  a point $x_{h,i} \in \mc{X}_h$ is evaluated $n_{h,i}$ times before a call to \refine is made. Then we must have that 
\begin{align}
\label{eq:nhi}
    n_{h,i} = \left\lceil \frac{ \beta_n^2 \tau^2}{L^2 v_1^{2\xi} \rho^{2h\xi}} \right \rceil. 
\end{align}
This is due to the following fact: suppose that at time $t$, the point $x_{h,i}$ has been evaluated $s$ times. Then by \citep[Proposition~3]{shekhar2018gaussian} we know that the posterior standard deviation at $x_{h,i}$ must satisfy $\sigma_t(x_{h,i}) \leq \tau/\sqrt{s}$. Plugging $s \leftarrow n_{h,i}$ from~\eqref{eq:nhi} in this bound implies that after $n_{h,i}$ evaluations, the condition $\beta_t \sigma_{t}(x_{h,i}) < L(v_1 \rho^h)^{\xi}$ is satisfied, and hence the point $x_{h,i}$ will not be evaluated anymore. Furthermore, we also know that if $\mc{P}_t \subset \mc{X}_h$, then it also satisfies the following two properties: \tbf{(i)} $\mc{P}_t \subset \mc{Z}_h \defined \left \{x \in \mc{X}: f(x^*) - f(x) \leq \lp 7 L v_1^\xi \rho^{-\xi} \rp \rho^{h\xi} \right\}$, and \tbf{(ii)} any two points in $\mc{P}_t$ are separated by a distance of $2v_2 \rho^{h}$. Together, these two facts imply that $\mc{P}_t$ must be a packing set of $\mc{Z}_h$, and thus we can upper bound $|\mc{P}_t|$ with $m_h$, the $2v_2\rho^h$ packing number of $\mc{Z}_h$. As a consequence, we have $N_h \leq n_{h,i} m_h = \tOh{ \rho^{-2h\xi} \rho^{-\dopt} }$.  
\end{proof}

It remains to show that the expected regret of $\mc{A}_1$ is upper bounded by the upper-complexity term $\compupper$. 

\begin{lemma}
\label{lemma:regret_bound3} 
    Introduce the term $H_n = \max\{H: \sum_{h=0}^H \rho^{-2h \xi} m_h \leq n\}$, and define $\Delta_n = \min \{n^{-(1-a_{\nu}^*)}, \rho^{H_n \xi} \}$. Then, we have 
    \begin{align}
        \mbb{E}[\mc{R}_n \lp \mc{A}_1, f \rp] = \tOh{\compupper(\Delta_n)}. 
    \end{align}
\end{lemma}

\begin{proof}
    We will work under the event $\mc{E}$ introduced in~\eqref{eq:event_conc}, that occurs with probability at least $1-1/n$. 
    The proof of this result follows by employing the upper bound on $N_h$ derived in~\Cref{lemma:regret_bound2}, and rearranging the resulting terms to form the upper-complexity term. 
    
In particular, we note that $\mc{R}_n \lp \mc{A}_1, f\rp \leq \sum_{h = 0}^{H_n}\tOh{ \rho^{h \xi } N_h}$. Now, due to the upper bound on $N_h$ obtained in~\Cref{lemma:regret_bound2}, we have $\mc{R}_n \lp \mc{A}_1, f \rp = \tOh{ \sum_{h= 0}^{H_n}\rho^{-h\xi} m_h}$. 
To rewrite this in terms of $\compupper$, note that $\Delta_n \leq \rho^{H_n \xi}$, and for $k \geq 0$, define 
\begin{align}
\widetilde{\mc{Z}}_k &= W_{H_n-k} = \{x: f(x^*) - f(x) \leq ( 7 L v_1^\xi \rho^{-\xi} ) \rho^{(H_n-k)\xi} \} \\
& = \{x: f(x^*) - f(x) \leq ( 7 L v_1^\xi \rho^{-\xi} ) \rho^{-\xi k} \rho^{H_n\xi} \} \\
& = \{x: f(x^*) - f(x) \leq ( 7 L v_1^\xi \rho^{-\xi} ) \rho^{-\xi k} \Delta_n \}. 
\end{align}
Having defined $\widetilde{\mc{Z}}_k$, we note that the term $\widetilde{m}_k$ in the definition of $\compupper$ is the same as $m_{H_n-h} = m\lp \widetilde{\mc{Z}}_k, 2v_2 \rho^{H_n-k} \rp$-packing number of $\widetilde{\mc{Z}}_k$. Finally, noting that $\rho^{H_n-k}\xi = (\rho^{-\xi k}) \Delta_n$, we have the following: 
\begin{align}
    \sum_{h=0}^{H_n} \rho^{-\xi h} m_h = \sum_{k=0}^{H_n}  \rho^{-\xi (H_n - k)} \widetilde{m}_k  = \sum_{k=0}^{H_n} \frac{\widetilde{m}_k}{ (1/\rho^{\xi})^k \Delta_n} = \compupper(\Delta_n). 
\end{align}
This completes the proof. 
\end{proof}

\section{Proof of~\Cref{theorem:bead_upper}}
As shown in~\eqref{eq:regret-event}, it suffices to get the bound on the regret under the $1-1/n$ probability event $\mc{E}$ introduced in~\eqref{eq:event_conc}. 
When the objective function, $f$, satisfies the local growth condition with exponent $b$, we can show that the regions $W_h = \{x : f(x^*) - f(x) \leq (7 L v_1^{\xi} \rho^{- \xi}) \rho^{h \xi} \}$ is contained in a ball centered at the optimal point $x^*$. In particular, due to the local-growth condition, it follows that $W_h \subset B\lp x^*, c' \rho^{h \xi/b} \rp$ for some constant $c'$. As $m_h = m \lp W_h, 2 v_2 \rho^{h} \rp$ is the $2v_2 \rho^h$-packing number of the set $W_h$, we can bound it from above by using volume arguments to get that $m_h = \lp \rho^{h d(1-\xi/b)} \rp$.

Combining this with the result of~\Cref{lemma:regret_bound2}, we get that 
\begin{align}
    \mc{R}_n \lp \mc{A}_1, f \rp = \tOh{ \sum_{h=0}^{h_{\max}} \rho^{-h \xi} \rho^{-hd(1-\xi/b)}}. 
\end{align}
Introducing the term $\widetilde{d} \defined d(1-\xi/b)$, we have $\mc{R}_n \lp \mc{A}_1, f \rp = \tOh{ \sum_{h=0}^{h_{\max}} \rho^{-h(\xi + \widetilde{d})}}$. Proceeding as in the proof of~\Cref{lemma:regret_bound3}, introduce the term $\widetilde{H}_n = \max \{H \leq h_{\max} : \sum_{h=0}^{H} \rho^{-h( 2\xi +\widetilde{d})} \}$, we see that $\rho^{-\widetilde{H_n}} = \tOh{ n^{1/(2\xi + \widetilde{d})}}$. This gives us the following: 
\begin{align}
    \mc{R}_n \lp \mc{A}_1, f \rp = \tOh{ \sum_{h=0}^{H_n} \rho^{-h(\xi+\widetilde{d})} + \rho^{H_n \xi} n } = \tOh{ n^{ (\xi + \widetilde{d}))/(2\xi + \widetilde{d})}}. 
\end{align}

\section{Extensions}
\label{appendix:extensions}
We first consider the Lipschitz bandit problem. Here, the goal is to design an adaptive querying strategy  to optimize an unknown $L$-Lipschitz objective function $f$ via noisy zeroth-order queries. For this problem, we can prove an analog of~\Cref{theorem:general_lower-1}.

\begin{definition}
\label{def:lipschitz-complexity}
Let $f$ be a $(1-\lambda) L$-Lipschitz function for some $\lambda \in (0,1)$. Fix a $\Delta>0$, and introduce the set $\mc{Z}_k \defined \{ x\in \X: 2^k \Delta \leq f(x^*) - f(x) < 2^{k+1}\Delta\}$. Introduce the radius $w_k = 3 \times 2^k \Delta/(\lambda L)$, and let $m_k$ denote the $2w_k$ packing  number of the set $\mc{Z}_k$ for $k \geq 0$. Then, we can define the following complexity term: 
        \begin{align}
            \label{eq:lipschitz-complexity} 
            \compLip(\Delta,  L, \lambda) \defined \sum_{k \geq 0} \frac{m_k}{2^{k+2}\Delta} > \frac{m_0}{4\Delta}. 
        \end{align}
\end{definition}

Then, proceeding as in the proof of~\Cref{theorem:general_lower-1}, we can obtain the following lower bound. 

\begin{proposition}
\label{prop:lipschitz-regret}
    For a $(1-\lambda)L$-Lipschitz function $f$, the expected regret of an $a_0$-consistent~(for the family of $L$-Lipschitz functions) algorithm $\mc{A}$ satisfies: 
    \begin{align}
        \mbb{E} \lb \mc{R}_n \lp \mc{A}, f\rp \rb = \Omega \lp \sigma^2 \compLip \lp n^{-(1-a)}, L, \lambda \rp \rp.  
    \end{align}
\end{proposition}

\begin{proofoutline}
    The general steps involved in obtaining this statement are similar to those used in the proof of~\Cref{theorem:general_lower-1}. In particular, we use the bump function of the form $g(x) = \max\{0, L(1 - \|x\|) \}$. We can check that this function is $L$-Lipschitz and supported on the unit ball. 
    
\end{proofoutline}

\end{appendix} 
\end{document}